\documentclass[conference]{IEEEtran}
\IEEEoverridecommandlockouts
\usepackage{cite}
\usepackage{amsmath,amssymb,amsfonts}
\usepackage{hyperref}
\usepackage[ruled,vlined,linesnumbered]{algorithm2e}

\usepackage{multirow} 
\usepackage{booktabs}
\usepackage{diagbox}
\usepackage{tabularx}
\newcolumntype{Y}{>{\centering\arraybackslash}X}
\usepackage{graphicx} 
\usepackage{makecell}

\usepackage{lipsum}
\usepackage{setspace}
\makeatletter

\usepackage[left=0.75in,
            right=0.75in,
            top=0.75in,
            bottom=0.75in]{geometry}

\usepackage{xcolor}
\usepackage{amsthm}
\usepackage{graphicx}
\usepackage{textcomp}
\usepackage{xcolor}
\def\BibTeX{{\rm B\kern-.05em{\sc i\kern-.025em b}\kern-.08em
    T\kern-.1667em\lower.7ex\hbox{E}\kern-.125emX}}

\newtheorem{observation}{Observation}
\newtheorem{theorem}{Theorem}

\newtheorem{lemma}{Lemma}

\usepackage{subfig}  
\usepackage{stfloats}

\usepackage[T1]{fontenc}
\usepackage{tikz}
\usepackage{ctable}

\begin{document}
\title{\vspace{0.21in} Asymptotically Optimal Lazy Lifelong Sampling-based Algorithm for Efficient Motion Planning in Dynamic Environments\\}

\author{
\IEEEauthorblockN{Lu Huang, \textit{Graduate Student Member, IEEE}, Jingwen Yu, \textit{Graduate Student Member, IEEE},}
\IEEEauthorblockN{
Jiankun Wang, \textit{Senior Member, IEEE},
Xingjian Jing, \textit{Senior Member, IEEE}}

\thanks{Lu Huang and Xingjian Jing are with Department of Mechanical Engineering, City University of Hongkong, Tat Chee Avenue, Kowloon, Hong Kong SAR.
{\tt\small (e-mail: \{lhuang98-c@my., xingjing@\}cityu.edu.hk)}}%

\thanks{Jingwen Yu and Jiankun Wang are with the Department of Electronic and Electrical Engineering, Southern University of Science and Technology, China.
{\tt\small (e-mail: jyubt@connect.ust.hk, wangjk@sustech.edu.cn)}}

\thanks{Jingwen Yu is also with the Department of Electronic and Computer Engineering, Hong Kong University of Science and Technology, Kowloon, Hong Kong SAR.}
}

\maketitle

\begin{abstract}
The paper introduces an asymptotically optimal lifelong sampling-based path planning algorithm that combines the merits of lifelong planning algorithms and lazy search algorithms for rapid replanning in dynamic environments where edge evaluation is expensive. 
By evaluating only sub-path candidates for the optimal solution, the algorithm saves considerable evaluation time and thereby reduces the overall planning cost. 
It employs a novel informed rewiring cascade to efficiently repair the search tree when the underlying search graph changes. 
Theoretical analysis indicates that the proposed algorithm converges to the optimal solution as long as sufficient planning time is given.
Planning results on robotic systems with $\mathbb{SE}(3)$ and $\mathbb{R}^7$ state spaces in challenging environments highlight the superior performance of the proposed algorithm over various state-of-the-art sampling-based planners in both static and dynamic motion planning tasks.
The experiment of planning for a Turtlebot 4 operating in a dynamic environment with several moving pedestrians further verifies the feasibility and advantages of the proposed algorithm.
\end{abstract}

\def\abstractname{Note to Practitioners}
\begin{abstract}

Fast and adaptive replanning is essential for robotic systems operating in dynamic environments with limited information and computational resources. 
While sampling-based planners are well-suited to such settings due to their ability to incrementally construct paths, their performance is often hindered by costly edge collision checking.
In this work, we propose a lifelong sampling-based planning algorithm that maintains and reuses a global roadmap over time, integrating a lazy edge evaluation strategy that defers collision checks until necessary. This approach enables efficient, continual replanning without rebuilding the search structure from scratch. Simulations and experiments in dynamic scenarios demonstrate improved planning speed and responsiveness, making the method suitable for real-time robotic applications.
\end{abstract}


\section{Introduction}
Autonomous robots frequently operate in environments that are partially or entirely unknown. For instance, a service robot navigating a crowded shopping mall or pedestrian area, a mobile robot using sensors to move outdoors without a complete roadmap, or multiple warehouse robots collaborating without predefined trajectories. In such scenarios, robots must quickly update their solutions based on new sensory data to ensure robust and responsive autonomy.

Early research proposed that reusing the previous searches can significantly enhance planning efficiency \cite{ramalingam1996incremental, Frigioni2000fully}. 
The strategy, commonly known as lifelong planning in literature, was combined with graph-based path planning techniques, such as A$^*$ \cite{Astar}.
In particular, Lifelong Planning A$^*$ (LPA$^*$) \cite{LPA} and its variants D$^*$/D$^*$ Lite \cite{DLite} repair the same search graph that approximates the problem domain throughout the entire navigation process.
Any inconsistency resulting from edge changes in the graph is identified, and the inconsistencies are efficiently propagated to the relevant parts of the search graph to repair an A$^*$-like solution. 
It ensures efficient adaptation to the dynamic nature of the problem domain. 
However, LPA$^*$ is only resolution-optimal and -complete with respect to the particular underlying graph used.
In other words, they almost surely find sub-optimal paths with respect to the robot and environment. 
Improving completeness/optimality requires recomputing the search graph, which is infeasible in applications.

\begin{figure*}[ht]
    \centering
      \subfloat[]{\includegraphics[width = 0.225\textwidth]{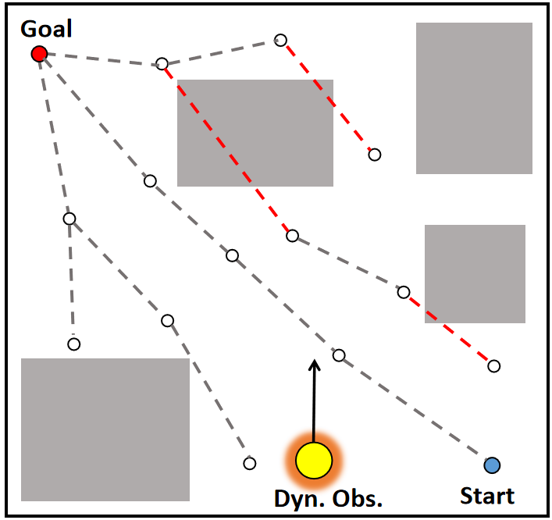}} \hfill
      \subfloat[]{\includegraphics[width = 0.225\textwidth]{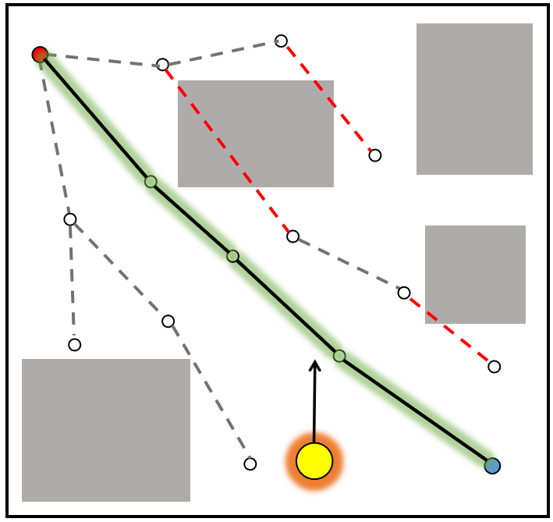}} \hfill
      \subfloat[]{\includegraphics[width = 0.225\textwidth]{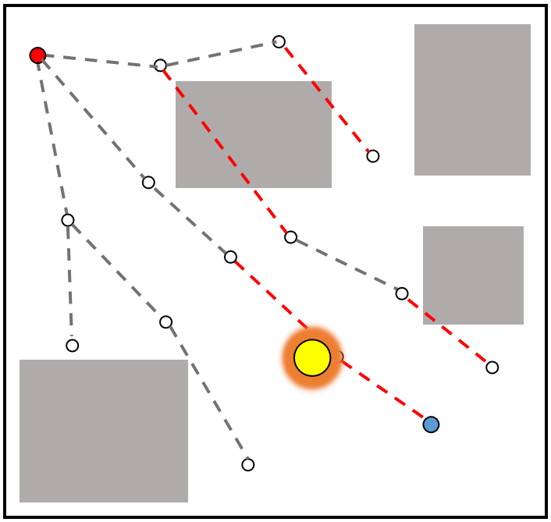}} \hfill
      \subfloat[]{\includegraphics[width = 0.225\textwidth]{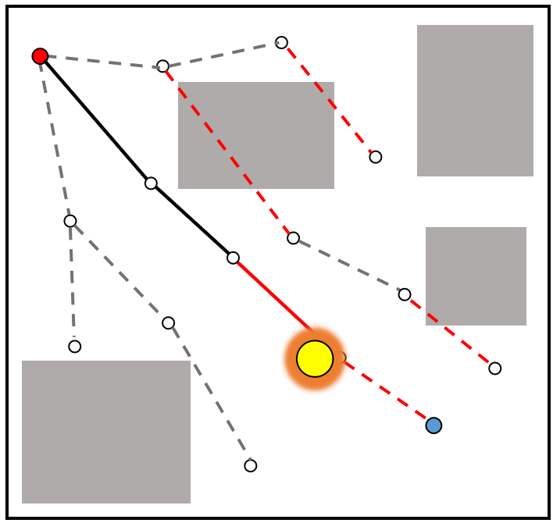}}

      \subfloat[]{\includegraphics[width = 0.225\textwidth]{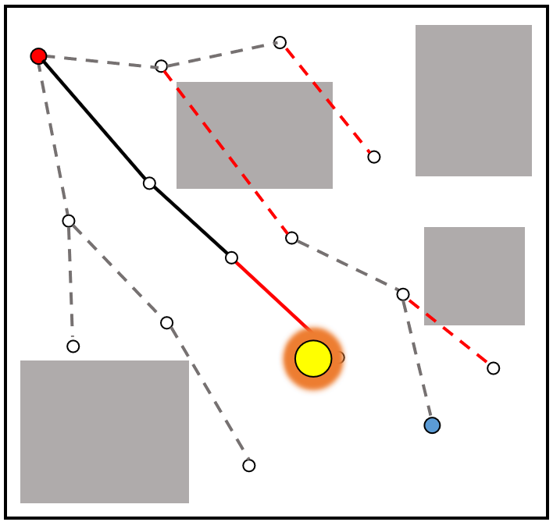}} \hfill
      \subfloat[]{\includegraphics[width = 0.225\textwidth]{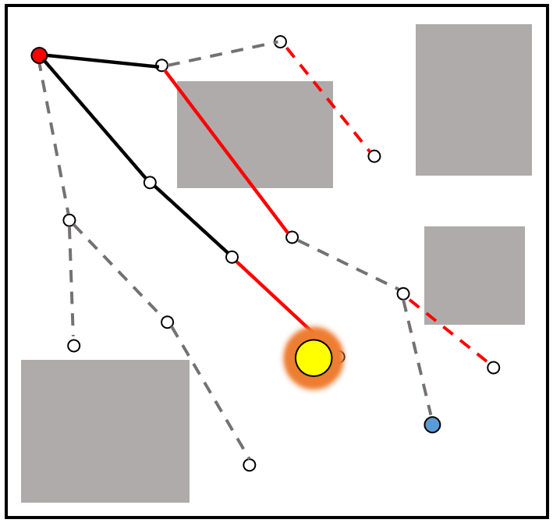}} \hfill
      \subfloat[]{\includegraphics[width = 0.225\textwidth]{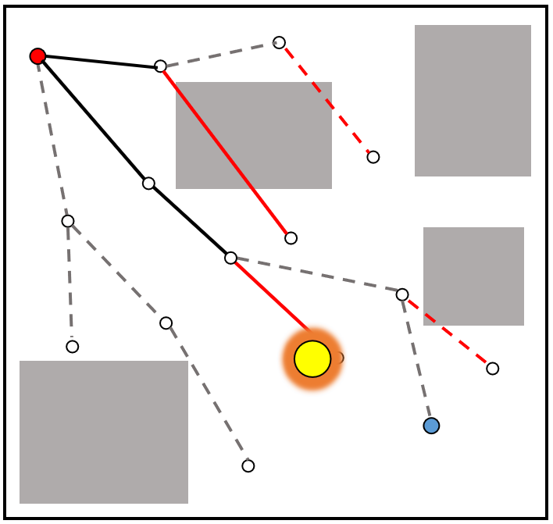}} \hfill
      \subfloat[]{\includegraphics[width = 0.225\textwidth]{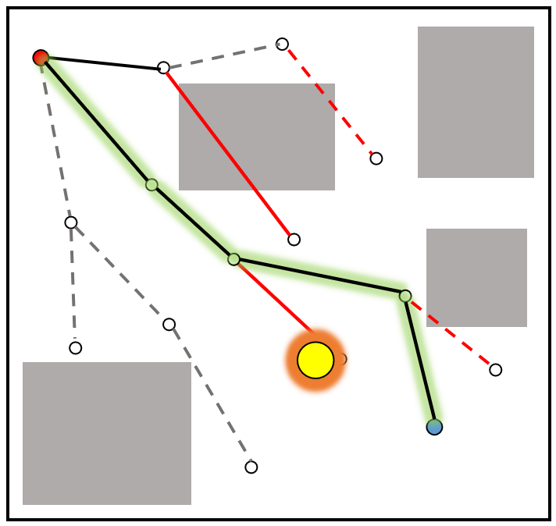}}
\caption{Graphical illustrations of LLPT$^*$ finding the shortest path from start to goal in the presence of a dynamic obstacle.
The gray dashed lines represent unevaluated but potentially valid edges, while the red dashed lines denote unevaluated and potentially invalid edges.
The black lines indicate evaluated and valid edges, and the solid red lines denote evaluated and invalid (i.e., colliding) edges.
(a) The configuration space is initially spanned by a search tree composed of unevaluated edges.
(b) A solution path from start to goal is found and verified to be collision-free.
(c) A dynamic obstacle moves, potentially altering the collision status of nearby edges.
(d) The edges along the previously found solution path are re-evaluated, and a collision is detected on one of them.
(e) The search tree is incrementally rewired, and a new solution path is identified.
(f) The new path is evaluated, and another colliding edge is discovered.
(g) The tree is rewired again, and the planner falls back to a suboptimal solution with higher cost.
(h) All edges along the new solution are finally evaluated and found to be collision-free. The solution is updated accordingly.
}
\label{fig:graphic_abs}
\end{figure*}

Asymptotically-optimal sampling-based planning methods, such as Rapidly-exploring Random Tree$^*$ (RRT$^*$) \cite{RRTstar}, RRT$^{\sharp}$ \cite{RRTsharp}, Informed RRT$^*$ \cite{InformedRRT}, in contrast, avoid an explicit offline construction of the problem domain and refine an asymptotically dense space-filling tree directly in the robot’s state-space. 
These methods aim to find asymptotically optimal paths in terms of the robot and its environment. 
However, a significant drawback of these methods is their inability to efficiently handle environment changes, as they require replanning from scratch whenever a change is detected.
To address this limitation, previous works have explored reusing previous search trees to improve the replanning efficiency of sampling-based planning methods \cite{DRRT, GRIP, LRF, MPRRT, RRTFND}. 
Unfortunately, these methods do not guarantee the quality of their solutions.
In contrast, asymptotically optimal lifelong sampling-based methods \cite{RTRRT, RRTX, FAT} stand out from previous approaches by offering asymptotically optimal solutions. 
They use a rewiring cascade to promptly remodel the search tree around dynamic obstacles.
These methods offer significant benefits, particularly in applications that require high solution quality.

Unfortunately, the existing asymptotically optimal lifelong sampling-based methods primarily focused on enhancing solution quality but need to be more indifferent to the number of edge evaluations.
Specifically, they must iteratively evaluate all edges of the underlying search graph to check their feasibility before rewiring, irrespective of their relevance to the solution.
In motion planning problems, edge evaluation typically consists of computationally-intensive procedures, such as solving two-point boundary value problems and conducting dynamic collision checks. 
As a result, the cost of edge evaluation dominates the overall replanning process, especially for motion planning problems in high-dimensional configuration spaces, where the underlying search graph must be dense enough to contain a solution.

The lazy search technique has been introduced in \cite{LazyPRM} to mitigate the problem of excessive edge evaluations. 
It assumes that a heuristic function can efficiently compute a lower bound on edge weight and uses the lower bound to estimate whether an edge has the potential to improve the solution.
This approach defers the evaluation of the actual cost of an edge until it is deemed pertinent to the solution. 
Various sampling-based algorithms \cite{LazyRRT, LBTRRT, BIT} have exploited the lazy search strategy to enhance their planning efficiency. 
Regrettably, none can achieve both the replanning efficiency of lifelong planning techniques and the edge evaluation efficiency of lazy search techniques.

In this paper, we propose an asymptotically optimal lifelong sampling-based algorithm, Lazy Lifelong Planning Tree$^*$ (LLPT$^*$), for online motion planning in environments where edge evaluation is expensive. 
The algorithm integrates the lazy search and the lifelong planning frameworks. 
Specifically, LLPT$^*$ approximates the problem domain by a Rapidly-exploring Random Graph (RRG) with non-overestimating heuristic edges. The RRG is reused and refined throughout the entire planning procedure until the robot reaches the goal. 
LLPT$^*$ searches for solutions by spanning the RRG with a Shortest-Path Tree (SPT). 
The actual edge evaluations of the SPT are restricted to those with the potential to be part of the solution. 
Whenever changes occur (e.g., the actual cost of some edges is evaluated or the RRG is refined), a novel informed rewiring cascade is performed to quickly repair the SPT to reflect the new information and update the solution.

This paper extends our previous work published at IROS 2024 \cite{LLPT_IROS}. We provide a comprehensive theoretical analysis, including an expanded proof of the probabilistic completeness and asymptotic optimality of LLPT$^*$ (see Section~\ref{Analysis}).
We also compare LLPT$^*$ with a broader range of sampling-based motion planning algorithms on various high-dimensional planning problems, including both static and dynamic scenarios, to further demonstrate its advantages (see Section~\ref{Simulation}).
The same section presents real-world experiments conducted on a Turtlebot 4 platform, validating the feasibility and efficiency of the proposed algorithm.
Furthermore, we evaluate the performance of LLPT$^*$ in highly dynamic environments and discuss potential directions for future improvement in the discussion part.


\section{Related Works}
\subsection{Lifelong Sampling-based Algorithms}
The concept of reusing previous search efforts is widely adopted by various planners to enhance their planning efficiency. 
Anytime motion planners \cite{RRT, InformedRRT, RRTconnect, RRTstar, RRTsharp} refine their solutions by adding new samples to existing search trees, leveraging previously found solution paths. 
However, these methods only reuse search efforts within a single query and must restart when the robot's operating environment or its start and goal configurations change.
Multi-query sampling-based planners, such as Probability Roadmap (PRM) \cite{PRM} and RRG$^*$ \cite{RRTstar}, address this limitation by efficiently solving multiple start-goal queries in static environments. They achieve this by reusing previously acquired knowledge of valid edges, which allows for more efficient solutions to subsequent queries.

Some sampling-based motion planners focus on reusing planning results to manage environmental variations effectively. For instance, certain approaches build trajectory libraries \cite{TrajLib} that store previous successful plans, guiding future searches based on handcrafted or learned heuristics, such as roadmap similarity \cite{IO-MP}, collision probability models \cite{PLUT}, or the proximity of start-goal configurations \cite{Lightning}. These trajectory libraries have proven effective for controlling underactuated and high-dimensional systems \cite{LittleDog, SoftRobot}. However, this approach can suffer from limited dataset sizes or excessive memory usage.
The Thunder framework \cite{Thunder} enhances trajectory library-based planners by storing experiences in a sparse graph rather than in individual paths. It reduces redundancy and increases opportunities for path reuse.

Additionally, significant efforts have been made to learn latent representations of configuration spaces from past experiences. This informs future queries by biasing sample sequences toward regions likely to contain a solution path \cite{WeightingFeature, HybridSampling, Bayesian}. 
Some approaches leverage neural networks to improve the accuracy of these latent representations \cite{Learning, CVAE, LEGO}. 
Others view biased sampling as a decision-making process, employing reinforcement learning methods to learn optimal sampling policies online \cite{MDP, RRF, MAB_kinoRRT}. 
However, these learning-based approaches often require extensive offline or online training to ensure the effectiveness of the sampler parameters.

Different from previous approaches, some accelerate replanning by rewiring the out-of-date search trees to fully exploit previous search efforts. 
DRRT and GRIP \cite{DRRT, GRIP} selectively eliminate invalid search tree branches obstructed by dynamic obstacles rather than discarding all previous samples as brute-force replanning. 
They then grow the search tree until a new solution is derived. 
More efficient replanning schemes only remove invalid edges of the spanning tree and preserve the separate branches for further planning. 
For instance, LRF retains a forest of spanning trees that are dynamically split, regrown, and merged to adapt to obstacles and robot movements \cite{LRF}. 
Similarly, MP-RRT retains obstructed branches and actively seeks to reconnect them to the root tree, bridging the gap between the start and goal configurations \cite{MPRRT}. 
RRT$^*$FND reduces replanning effort by only reconnecting or regrowing functional disconnected branches of the previous solution path \cite{RRTFND}.

However, the aforementioned approaches are lack of formal guarantees on solution optimality. 
Distinguished from previous works, RT-RRT$^*$\cite{RTRRT}, RRT$^X$\cite{RRTX}, and FAT$^*$\cite{FAT} were designed to address dynamic shortest-path planning problems (DSPP), offering an asymptotically optimal solution given sufficient replanning time.
These methods leverage rewiring cascades to adapt the search tree and navigate around dynamic obstacles. 
In particular, RT-RRT$^*$ rewires the entire search tree structure starting from the root, which can be computationally expensive when there are many nodes in the tree. 
On the other hand, RRT$^X$ uses a local rewiring cascade, which focuses only on rewiring the parts of the search tree affected by dynamic changes, rendering it more efficient in scenarios where the environment experiences minor alterations.
FAT$^*$ further optimizes the rewiring process of RRT$^X$ by constraining expansions to promising samples that are likely to improve solution costs.

It is worth noting that the local rewiring cascade can be enhanced in efficiency by ordering rewiring based on the potential
solution quality and constraining rewiring only to the informed node set relevant to the solution, akin to methodologies employed in informed graph-based search techniques (i.e., LPA$^*$, D$^*$).
In the later section, we propose an informed local rewiring cascade and integrate it into the proposed algorithm.

\subsection{Path Planning with Lazy Collision Detection}
The issue of expensive edge evaluations is particularly significant in robotics. 
While numerous studies aim to reduce collision checks by biasing sampling toward free regions \cite{Bial, SafeCertificate}, it is common for the motion between pairs of free samples to be infeasible, especially in environments with narrow passages or dense obstacles.
One effective strategy is to select paths that minimize the likelihood of collisions. 
Pan and Manocha develop incremental collision probability models by which only the paths that are estimated to be feasible are evaluated \cite{LSH}. 
This concept has also been explored within a Bayesian reinforcement learning framework \cite{ExperiencedLazyPathSearch}.
Choudhury et al. \cite{BISECT} frame edge selection as a decision region determination problem, choosing edges that reduce uncertainty about the robot's operating environment. 
STROLL \cite{STROLL} formulates the edge selection problem as a Markov Decision Process (MDP), learning an edge evaluation policy by imitating the strategy of LazySP \cite{LazySP}. 
Additionally, Sung and Stone \cite{GraphCut} prioritize the evaluation of edges located on the cut of the underlying graph, as these edges are more closely related to potential solutions.

On the other hand, some works assume that a heuristic function can estimate a lower bound on edge costs effectively and guide collision checks by estimated solution length.
Lazy PRM$^*$ and Lazy RRG$^*$ \cite{LazyRRT} build a roadmap of feasible configurations like PRM$^*$ and RRG$^*$ \cite{RRTstar}, but with the difference that edges are not immediately checked for collision.
Only edges that are estimated by the heuristic function to belong to a better candidate solution path are checked for collision. If an edge is found to collide, it is removed from the roadmap, and the candidate path is updated.
Similarly, Lower Bound Tree-RRT (LBT-RRT) \cite{LBTRRT} proposed the lower bound graph $\mathcal{G}_{lb}$, an RRG with lazily-evaluated edges. $\mathcal{G}_{lb}$ poses a lower bound invariant to the nodes in the search tree, i.e., the solution cost of the nodes in the search tree is always smaller than their solution cost via $\mathcal{G}_{lb}$. An edge is only evaluated if its lazy insertion into $\mathcal{G}_{lb}$ causes the violation against the lower bound invariant. 
Batch Informed Tree$^*$ (BIT$^*$) \cite{BIT} inserts a batch of samples concurrently into the search graph with lazily evaluated edges and prioritizes the evaluation of the new edges with minor lazily-evaluated solution cost. 

Our work brings together lifelong sampling-based algorithms and lazy search techniques to simultaneously harvest replanning and edge evaluation efficiency when solving DSPP. Our algorithm outperforms existing lifelong and lazy sampling-based algorithms in terms of planning efficiency for both static and dynamic planning problems.

\section{Lazy Lifelong Planning Tree$^{*}$}
\label{Methodology}

\begin{figure}
    \centering
    \includegraphics[width=0.49\textwidth]{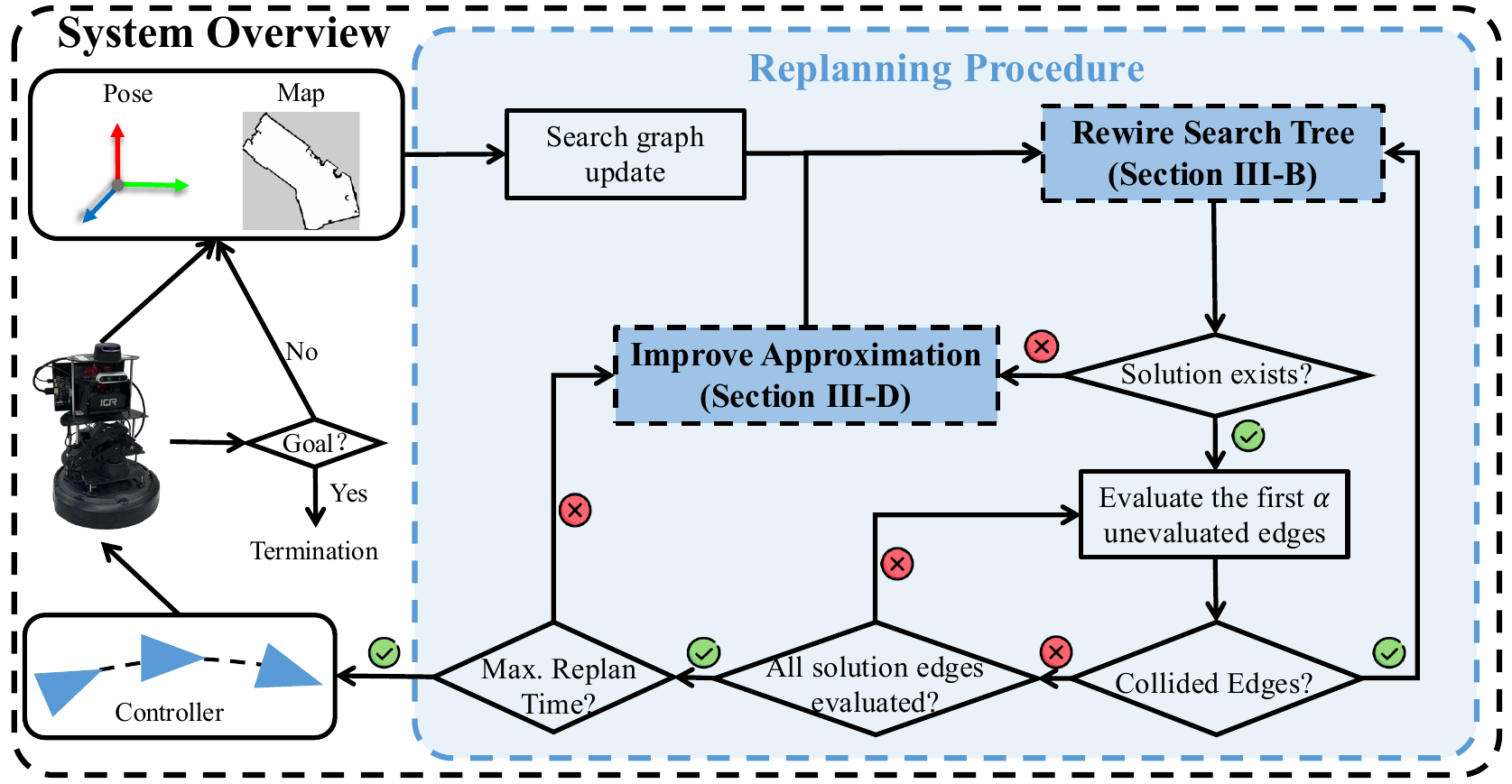}
    \caption{System overview of the proposed replanning module integrated with a complete robot navigation system.}
    \label{fig:system-overview}
\end{figure}

In this section, we describe the proposed lifelong planning framework. The system overview is presented in \ref{fig:system-overview}.
In the following subsection, we first introduce the notations that will be used throughout the paper. 
In Section \ref{InformedRewiringCascade}, we highlight the informed rewiring cascade, a crucial component of LLPT$^*$ that efficiently rewires the underlying graph upon detecting graph modifications. 
We then present the overall algorithm framework, with a particular focus on the lazy search strategy that is used to delay collision checks. 
Finally, we describe the graph densification strategy used to enhance the solution.

\subsection{Notations and Problem Definitions}
Assume that a robot begins its journey from $v_{start}$ to $v_{goal}$. 
An optimal path planning problem is to find a path $\pi^*(v_{start},v_{goal})$, an ordered set of distinct vertices that connects $v_{start}$ and $v_{goal}$ and minimizes the total cost to traverse this path.
LLPT$^*$ solves this problem by constructing a search graph $\mathcal{G} = (V, E)$ with lazily-evaluated edges within the robot's configuration space $\mathcal{X}$, where $V$ represents the node set and $E \subseteq V \times V$ denotes the edge set. To clarify, each node $v \in V$ implies the corresponding configuration $x \in \mathcal{X}$ and can be directly input into distance functions. 
For each node, LLPT$^*$ stores its neighbor nodes in memory for future reference.
The search tree of $\mathcal{G}$ is denoted by $\mathcal{T} = (V_{\mathcal{T}}, E_{\mathcal{T}})$, rooted at $v_{goal}$, where $V_{\mathcal{T}} \subseteq V$ and $E_{\mathcal{T}} \subseteq E$.  
As mentioned in \cite{RRTX}, rooting $\mathcal{T}$ in $v_{goal}$ allows the same tree and root
to be used for the entire navigation, avoiding rewiring the tree from the root while the robot moves.

For each edge $e\in E$ of $\mathcal{G}$, a weight function $w: E \rightarrow (0, \infty]$ assigns a positive real number, including infinity, to this edge, e.g., the distance/time/energy cost for traversing this edge. If traversing the edge is infeasible, the edge weight will be infinity. 
Denote the cost of traversing the local trajectory of an edge in the free space as $\hat{w}(e)$, and $E_{eval}$ be the set of evaluated edges at the current replanning iteration, the edge weight of $\mathcal{G}$ is defined by a lazy weight function
\begin{equation}
\label{EdgeWeightUpdate}
    \bar{w}(e)
    =\begin{cases}
         w(e), & \text{if } e\in E_{eval},\\
         \hat{w}(e),       & \text{else.}
    \end{cases}
\end{equation}
The goal of our planner is to find the optimal solution path $\pi^*$ that minimizes the total true edge weight, i.e., $\pi^*(v_{start},v_{goal})=\arg\min_{\pi\in\mathcal{G}}\sum_{e\in\Pi}w(e)$, where $\Pi$ is the set of all finite cost paths connecting $v_{start}$ to $v_{goal}$ in $\mathcal{G}$.


\begin{algorithm}[t]
\setstretch{1.0}
  \caption{Lazy Lifelong Planning Tree$^{*}$}
  \label{Main}
  \SetKwFunction{TimePermits}{TimePermits}
  \SetKwFunction{PropagateAscendants}{PropagateAscendants}
  \SetKwFunction{ReduceInconsistency}{ReduceInconsistency}
  \SetKwFunction{EvaluateEdge}{EvaluateEdge}
  \SetKwFunction{Selector}{Selector}
  \SetKwFunction{ComputeShortestPath}{ComputeShortestPath}
  \SetKwFunction{ShrinkRadius}{ShrinkRadius}
  \SetKwFunction{Parent}{Parent}
  \SetKwFunction{ExtendSearchGraph}{ExtendSearchGraph}
  \SetKwFunction{CheckObstacleRemoval}{CheckObstacleRemoval}
  \SetKwFunction{Child}{Child}
  \SetKwFunction{InformCollision}{InformCollision}
  \SetKwFunction{RemoveFromTree}{RemoveFromTree}
  \SetKwFunction{PropagateCostToLeave}{PropagateCostToLeave}
  \SetKwFunction{update}{update}
  \SetKwFunction{UpdateNode}{UpdateNode}
    $E,E_{\mathcal{T}}\leftarrow\phi$; 
    $V\leftarrow\{v_{goal}, v_{start}\}$,$V_{\mathcal{T}}\leftarrow\{v_{goal}\}$ \; 
    \label{Main:InitGoal}    
    
    $\mathcal{Q}\leftarrow\phi$, $k_{m}=0$, $\pi\leftarrow\phi$\;
    
    \While{$v_{start}\neq v_{goal}$}{
      \label{Main:ReplanBegin}
      Update the environment information\;
      
      $E_{eval}\leftarrow\phi$, $E_{collided}\leftarrow\phi$\;
      
      Update the edge weight of $\mathcal{G}$ based on Eq.\ref{EdgeWeightUpdate}\; \label{Main:Refresh}

      $E_{modified}\leftarrow$ the set of edges with modified weight\; \label{Main:RefreshUpdateBegin}

      \For{$e=(v, u) \in E_{modified}$}
      {
        \UpdateNode($v,u$)\;
        \label{Main:RefreshUpdateEnd}
      }
      
      \While{\TimePermits{}}
      {
        \label{Main:OuterLoopBegin}
        \Repeat{$\pi\subset E_{eval}$ \textbf{ and } $\pi\cap E_{collided}=\phi$}
        {
            \label{Main:InnerLoopBegin}
              $\pi\leftarrow$\ComputeShortestPath()\; \label{Main:ComputeShortestPath}
              \eIf{$v_{start},v_{goal}\in\pi$}
              {
                $V_{collided}=$\EvaluateEdge{$\alpha$, $\pi$}\; \label{Main:Evaluation}
                \For{$v\in V_{collided}$}
                { 
                    \PropagateCostToLeave($v$)\;  
                    \label{Main:PropagateCost}
                }
                \label{Main:PropagateCostEnd}
              }
              {
                break \tcp*{No solution exist.} \label{Main:NoSolution} 
              }   
            }\label{Main:InnerLoopEnd}
            \ExtendSearchGraph()\; \label{Main:AddNewSample}
            \label{Main:OuterLoopEnd}
        }
      
        \If{robot is moving}{
            $v_{last}\leftarrow v_{start}$\;
            Update $v_{start}$\;
            $k_m \leftarrow k_m + w(v_{last}, v_{start})$\;
            \label{Main:Transient}
        }
        \label{Main:ReplanEnd}
      }

\end{algorithm}

\begin{algorithm}[t]
\setstretch{1.0}
  \caption{EvaluateEdge($\pi$, $\alpha$)}
  \label{EvaluateEdge}
  \SetKwFunction{RemoveFromTree}{RemoveFromTree}
    $E_{uneval}\leftarrow$ unevaluated edges in $\pi$\;
    $V_{collided}\leftarrow\phi$\;
    \If{$|E_{uneval}| > \alpha$}
    {
    \label{Evaluation:EdgeSelectBegin}
        $E_{uneval}\leftarrow$ first $\alpha$ edges in $E_{uneval}$ closest to $v_{goal}$\;
    \label{Evaluation:EdgeSelectEnd}
    }
    \For{$e(v,u)\in E_{uneval}$}{
    \label{Evaluation:EvalBegin}
        \If{$e$ is collided with some obstacles}
        {
            $E_{collided}\xleftarrow{+}\{e\}$, $V_{collided}\xleftarrow{+}\{v\}$\;
        }
        $E_{eval}\xleftarrow{+}\{e\}$\;
    \label{Evaluation:EvalEnd}
    }
    \textbf{return} $V_{collided}$\;
\end{algorithm}

\begin{algorithm}[t]
\setstretch{1.0}
  \caption{PropagateCostToLeave($v)$}
  \label{PropagateCostToLeave}
  \SetKwFunction{PropagateCostToLeave}{PropagateCostToLeave}
  \SetKwFunction{Child}{Child}
  \SetKwFunction{RemoveFromTree}{RemoveFromTree}
  \SetKwFunction{update}{update}
   \RemoveFromTree{$v$}\;\label{PropagateCostToLeave:CollidedNodeProcessBegin}
   $lmc(v)\leftarrow\infty$, $\mathcal{Q}$.\update{$v$}\;\label{PropagateCostToLeave:CollidedNodeProcessEnd}
    \For{$u\in$ \Child{$v$}}{
        \label{PropagateCostToLeave:ToChildBegin}
        \PropagateCostToLeave{$u$}\;
        \label{PropagateCostToLeave:ToChildEnd}
    }
\end{algorithm}

\begin{algorithm}[t]
\setstretch{1.0}
  \caption{ComputeShortestPath()}
  \label{ComputeShortestPath}
  \SetKwFunction{TimePermits}{TimePermits}
  \SetKwFunction{FindParentOf}{FindParentOf}
  \SetKwFunction{MakeParentOf}{MakeParentOf}
  \SetKwFunction{topnode}{topnode}
  \SetKwFunction{Neighbor}{Neighbor}
  \SetKwFunction{Parent}{Parent}
  \SetKwFunction{pop}{pop}
  \SetKwFunction{key}{key}
  \SetKwFunction{topkey}{topkey}
  \SetKwFunction{CalculateKey}{CalculateKey}
  \SetKwFunction{update}{update}
  \SetKwFunction{UpdateNode}{UpdateNode}
  \SetKwFunction{RewireNeighbor}{RewireNeighbor}
    \While{$|\mathcal{Q}|>0$ 
    \textbf{ and } $(lmc(v_{start})>g(v_{start})
    $\textbf{ or } $lmc(v_{start})=g(v_{start})=\infty$
    \textbf{ or } $v_{start}\in\mathcal{Q}$
    \textbf{ or } $\mathcal{Q}$.\topkey{} $<$ \CalculateKey{$v_{start}$}}{
        \label{ReduceInconsistency:Terminate}
        $k_{old}\leftarrow\mathcal{Q}$.\topkey{}\; \label{ReduceInconsistency:UpdatePriorityBegin}
        $v\leftarrow\mathcal{Q}$.\pop{}\;
        \eIf{$k_{old}<$\CalculateKey{$v$}}
        {
            $\mathcal{Q}$.\update{$v$}\; \label{ReduceInconsistency:UpdatePriorityEnd}
        }
        {
            \If{$lmc(v)>g(v)$\textbf{ or }$lmc(v)=g(v)=\infty$}{ \label{ReduceInconssistency:UnderconsistentStart}
                \For{$u\in V_{\mathcal{T}}\cap N(v)$}
                {
                    \If{$lmc(v)>\bar{w}(v,u)+lmc(u)$}
                    {
                        \MakeParentOf{$v,u$}\;
                        $lmc(v)\leftarrow \bar{w}(v,u)+lmc(u)$\;
                    }
                } 
            }\label{ReduceInconssistency:UnderconsistentEnd}
            
            \If{$lmc(v)\neq\infty$}{   
              \label{ReduceInconssistency:RewireCondition}
              \For{$u\in N(v)\backslash$ \Parent{$v$}}
              { 
                \label{ReduceInconssistency:OverconsistentBegin}
                \label{ReduceInconssistency:RewireNeighborBegin}
                \If{$lmc(u)>\bar{w}(u,v)+lmc(v)$}
                { 
                  \MakeParentOf{$u,v$}\;
                  $lmc(u)\leftarrow \bar{w}(u,v)+lmc(v)$\;                    $\mathcal{Q}$.\update{$u$}\;\label{ReduceInconssistency:RewireNeighborEnd}
                }
              } 
            } \label{ReduceInconssistency:OverconsistentEnd}
        $g(v)\leftarrow lmc(v)$;\label{ReduceInconsistency:Consistent}  
        }
    }
    \textbf{return} path from $v_{start}$ to $v_{goal}$\;
\end{algorithm}

\begin{algorithm}[h]
\setstretch{1.0}
  \caption{UpdateNode($v, u$)}
  \label{UpdateNode}
  \SetKwFunction{MakeParentOf}{MakeParentOf}
    \If{$lmc(v)>\bar{w}(v,u)+lmc(u)$}
    {
        \MakeParentOf{$v$,$u$}\;   
        $lmc(v)\leftarrow \bar{w}(v,u)+lmc(u)$\;
        $\mathcal{Q}$.\update{$v$}\;  
    }
\end{algorithm}

\begin{algorithm}[h]
\setstretch{1.0}
  \caption{ExtendSearchGraph()}
  \label{ExtendSearchGraph}
  \SetKwFunction{RandomSample}{RandomSample}
  \SetKwFunction{Nearest}{Nearest}
  \SetKwFunction{Steer}{Steer}
  \SetKwFunction{Neighbor}{Neighbor}
  \SetKwFunction{FindParentOf}{FindParentOf}
  \SetKwFunction{MakeParentOf}{MakeParentOf}
  \SetKwFunction{update}{update}
  \SetKwFunction{UpdateNode}{UpdateNode}
        $v_{rand} \leftarrow$ \RandomSample{}\; 
        \label{AddNewSample:RandomSample}
        $v_{nearest}\leftarrow$ \Nearest{$v_{rand}$}\; \label{AddNewSample:FindNearest}
        $v_{new}\leftarrow $\Steer{$v_{nearest},v_{rand}$, $\delta$}\; 
        \label{AddNewSample:Steer}
        \If{$v_{new}$ is in the free configuration space}
        {
            $V\xleftarrow{+}\{v_{new}\}$\;
        }
        \For{$u\in N(v_{new})$}
        {
            \label{AddNewSample:AddSampleStart}
            \If{free-space local trajectory exists between $v$ and $u$}
            {
                $E\xleftarrow{+}\{e(v_{new},u), e(u,v_{new})\}$\;
                \UpdateNode{$v, u$}\; \label{AddNewSample:UpdateNode}
            }
            \label{AddNewSample:AddSampleEnd}  
        }
        
\end{algorithm}

\subsection{Informed Rewiring Cascade}
\label{InformedRewiringCascade}
LLPT$^{*}$ dynamically adjusts the solution path by rewiring $\mathcal{T}$ when the edge weights of $\mathcal{G}$ change. Drawing inspiration from RRT$^X$ \cite{RRTX} and RRT$^\sharp$ \cite{RRTsharp}, LLPT$^{*}$ employs two cost-to-goal estimates, g($v$) and lmc($v$), for each node to detect local inconsistencies arising from graph modifications.
A node $v$ is called \textit{consistent} if g($v$) = lmc($v$)$<\infty$, and \textit{inconsistent} otherwise. 
An inconsistent node is further categorized as locally \textit{overconsistent} if lmc($v$) $<$ g($v$), or locally \textit{underconsistent} if lmc($v$) $>$ g($v$) or lmc($v$) = g($v$) =$\infty$.
The g-value is the accumulated cost-to-goal by traversing the previous
search tree, while the lmc-value is the minimal cost-to-goal based on the up-to-date graph topology and is potentially better informed than the g-value. The lmc-value satisfies the following relationship when all nodes are locally consistent:
\begin{equation}
    lmc(v)\leftarrow     
    \begin{cases}
     0, & \text{if }v=v_{goal},\\
     \min_{u\in N(v)} \bar{w}(v, u) + lmc(u),  & \text{otherwise}.
    \end{cases}
\end{equation}
Then the shortest path from $v_{start}$ to $v_{goal}$ can be found by traversing iteratively from $v_{start}$ to the neighbor node $u$ that minimizes $\bar{w}(v, u) + lmc(u)$.

Modifications on edge weights potentially trigger changes in the lmc-value of some nodes and render these nodes inconsistent. 
To minimally rewire the inconsistent nodes while guaranteeing solution completeness and optimality, we design the rewiring cascade based on the following observations.
\begin{observation}
\label{underconsistent}
If the cost of the solution path from a node to the goal increases, the cost-to-goal of all its descendants (if it has any) in $\mathcal{T}$ will implicitly increase by the exact change. 
The node and its descendants may find a shorter path to the goal among their neighboring nodes belonging to $\mathcal{T}$.
Therefore, the node and its descendants should be rewired to guarantee solution completeness after rewiring $\mathcal{T}$.
\end{observation}

\begin{observation}
\label{overconsistent}
If the cost of the solution path from a node to the goal decreases, the cost-to-goal of all its descendants (if it has any) in $\mathcal{T}$ will implicitly decrease by the exact change. 
The node and its descendants may provide a shortcut to their neighboring nodes. 
Therefore, all the neighboring nodes should be rewired to guarantee solution optimality.
\end{observation}

LLPT$^*$ collects and stores the inconsistent nodes in a priority queue $\mathcal{Q}$.
The priority of each node in the queue is based on the order of its associated keys.
If a node $v$ becomes underconsistent, which happens when its solution path to the goal is evaluated to be in collision, the node and all its leaf nodes should be rewired to guarantee solution completeness according to \textbf{Observation \ref{underconsistent}}.
LLPT$^{*}$ will use the \FuncSty{PropagateCostToLeave} procedure (Algorithm \ref{PropagateCostToLeave}) to propagate the cost ascendance to all its leaf nodes. 
Afterward, both $v$ and its leaf nodes will be added to the priority queue $\mathcal{Q}$. 
On the other hand, if a node $v$ becomes overconsistent, which can happen when some of its outgoing edges have decreasing weight or when $\mathcal{G}$ is extended, $v$ will be added to $\mathcal{Q}$ to trigger rewiring among its neighboring nodes to ensure solution optimality according to \textbf{Observation \ref{overconsistent}}. 

The $\FuncSty{ComputeShortestPath}$ procedure, which is provided in Algorithm \ref{ComputeShortestPath}, is called to make the spanning tree consistent again by operating on the inconsistent nodes stored in $\mathcal{Q}$ iteratively. 
For an underconsistent node, $\FuncSty{ComputeShortestPath}$ iteratively searches among its neighbor nodes, and the underconsistent node is rewired to a new parent node if the new parent node results in a better cost-to-goal (line \ref{ReduceInconssistency:UnderconsistentStart}-\ref{ReduceInconssistency:UnderconsistentEnd}, Algorithm \ref{ComputeShortestPath}).
For an overconsistent node, its local inconsistency is disseminated to its neighbors. 
Specifically, if $v$ emerges as a better parent node for a neighbor node, the neighbor's parent is updated to $v$, and the neighbor will be included in $\mathcal{Q}$ to propagate the cost-to-goal reduction (line \ref{ReduceInconssistency:OverconsistentBegin}-\ref{ReduceInconssistency:OverconsistentEnd}, Algorithm \ref{ComputeShortestPath}).

To save computational cost, we only operate on inconsistent nodes with the potential to be in the solution path. 
These nodes are called \textit{promising} in \cite{RRTsharp}. 
For a promising node $v$, an admissible cost estimate of the path from $v_{start}$ to $v_{goal}$ constrained to pass through it is smaller than the current best solution cost.
Define the function $h(v_{start}, v):v\rightarrow \mathbb{R}^{+}$ as an admissible heuristic estimate of the cost from $v_{start}$ to $v$, i.e., $h(v_{start}, v)\leq w(v_{start}, v)$, and the key of a node as $k(v)=[k_1(v), k_2(v)]$, where $k_1(v)=\min(lmc(v),g(v))+h(v)$ and $k_2(v)=\min(lmc(v),g(v))$, the definition of the promising node set $V_{prom}$ is given as follows:
\begin{equation}
    V_{prom}=\{v\in V| k(v)\dot{<}k(v_{start})\},
\end{equation}
where $\dot{<}$ is a lexicographical comparator with $k(v)\dot{<}k(u)$ if either $k_1(v) < k_2(u)$ or $k_1(v) = k_1(u) \land k_1(v) < k_2(u)$.
The node in $\mathcal{Q}$ with the smallest key is regarded as the most promising one, which is popped out from $\mathcal{Q}$ and processed first until there is no node with smaller key than $v_{start}$ and $v_{start}$ has found a solution path to $v_{goal}$ (line 1, Algorithm \ref{ComputeShortestPath}).
Since $v_{start}$ is updated as the robot moves, $h(v_{start},v)$ is changing and thus the key used by the priority queue. 
To avoid having to reorder the queue, we adopt the key updating strategy as D$^*$ Lite\cite{DLite}, where a constant $k_m$ (initialized as 0) that is added up with the cost of the robot movement at each replanning iteration is used to mitigate the priority inconsistency between the out-of-date priority in $\mathcal{Q}$ and the priority of a newly-inserted vertex (line \ref{Main:Transient}, Algorithm \ref{Main}). 
The key of a newly-inserted vertex $k(v)$ is calculated by
$k_1(v)=\min(lmc(v),g(v))+h(v_{start}, v)+  k_m$ and
$k_2(v)=\min(lmc(v),g(v))$. 
For each node popped from the top of $\mathcal{Q}$, its key will be recalculated, and the node will be reinserted into $\mathcal{Q}$ with updated priority if its old key is smaller than the new one (line \ref{ReduceInconsistency:UpdatePriorityBegin}-\ref{ReduceInconsistency:UpdatePriorityEnd}, Algorithm \ref{ComputeShortestPath}).  
This assures the efficiency of the queue order.

\subsection{Algorithmic framework}
The proposed algorithm is provided in Algorithm \ref{Main} with its major sub-routines outlined in Algorithms \ref{EvaluateEdge}-\ref{ExtendSearchGraph}. Notations such as $A\xleftarrow{+}B$ represents $A\leftarrow A\cup B$, and $A\xleftarrow{-}B$ represents $A\leftarrow A\backslash B$. 
Set cardinality is denoted by $|\cdot|$. 
$\FuncSty{Child}(v)$ and $\FuncSty{Parent}(v)$ retrieve child nodes and the parent of node $v$ in $\mathcal{T}$, respectively. 
$N(v)$ returns the neighboring nodes of $v$ in $\mathcal{G}$.
The priority queue $\mathcal{Q}$ stores nodes sorted in ascending order of their keys. 
$\mathcal{Q}.\FuncSty{update}(v)$ inserts or re-orders node $v$ based on its key value. 
$\mathcal{Q}.\FuncSty{pop}()$ pops and returns the node with the highest priority in $\mathcal{Q}$. 
$\mathcal{Q}.\FuncSty{topkey}$ returns the key of the vertex with the highest priority in $\mathcal{Q}$.
The function $\FuncSty{MakeParentOf}(v,u)$ establishes a child-parent relationship between $v,u$ and adds $v$ to $\mathcal{T}$.
$\FuncSty{RemoveFromTree}(v)$ eliminates the child-parent relationship between $v$ and its parent node and removes $v$ from $\mathcal{T}$. 
The comparison operator with $\infty$ is defined such that $x < \infty$ is true for any $x\neq\infty$ and false otherwise.

LLPT$^*$ initiates by setting the goal node $v_{goal}$ with lmc($v_{goal}$)=g($v_{goal}$)=0 and the start node $v_{start}$ with lmc($v_{start}$)=g($v_{start}$)=$\infty$. 
The replanning procedure (line \ref{Main:ReplanBegin}-\ref{Main:ReplanEnd}, Algorithm \ref{Main}) iterates indefinitely at a user-defined frequency until the robot reaches the goal configuration.
At the start of each replanning cycle, the search graph $\mathcal{G}$ is refreshed with lazily-computed edges (line \ref{Main:Refresh}, Algorithm \ref{Main}).
For edges with updated weights, their starting nodes are rewired using $\FuncSty{UpdateNode}$ (line \ref{Main:RefreshUpdateBegin}-\ref{Main:RefreshUpdateEnd}, Algorithm \ref{Main}).
Each replanning cycle comprises two primary loops: the outer loop (line \ref{Main:OuterLoopBegin}-\ref{Main:OuterLoopEnd}, Algorithm \ref{Main}) and the inner loop (line \ref{Main:InnerLoopBegin}-\ref{Main:InnerLoopEnd}, Algorithm \ref{Main}).
The inner loop serves as the main search process, rewiring the search tree $\mathcal{T}$ and identifying the shortest collision-free path.
This loop alternates between rewiring $\mathcal{T}$ to compute the shortest path and evaluating the edges along the shortest path until a collision-free shortest path is discovered.
The procedure \FuncSty{EvaluateEdge} evaluates the first $\alpha$ unevaluated edges of the shortest path, where $\alpha\in\mathbb{N}^*$ is a user-specified parameter to balance the competing computational costs of edge evaluation and rewiring, and returns the collided nodes.
For the collided nodes, they are removed from $\mathcal{T}$. 
Their lmc-value is set to infinity, and they are then inserted into $\mathcal{Q}$ (line \ref{PropagateCostToLeave:CollidedNodeProcessBegin}-\ref{PropagateCostToLeave:CollidedNodeProcessEnd}, Algorithm \ref{PropagateCostToLeave}). 
The cost-to-goal increase information is iteratively propagated to their descendants (line \ref{PropagateCostToLeave:ToChildBegin}-\ref{PropagateCostToLeave:ToChildEnd}, Algorithm \ref{PropagateCostToLeave}).
The inner loop continues to rewire the search tree and updates the solution (line \ref{Main:ComputeShortestPath}, Algorithm \ref{Main}).
If no solution for the current $\mathcal{G}$ is found, the inner loop reports failure and terminates (line \ref{Main:NoSolution}, Algorithm \ref{Main}). 
When the inner loop terminates with remaining planning time, the search graph $\mathcal{G}$ is densified in the outer loop to enhance potential solution quality (line \ref{Main:AddNewSample}, Algorithm \ref{Main}).

\subsection{Search Graph Densification}
The \FuncSty{ExtendSearchGraph} procedure (Algorithm \ref{ExtendSearchGraph}) extends the underlying search graph $\mathcal{G}$ to improve both completeness and optimality of the solution.
$\mathcal{G}$ is growing to incrementally approximate the problem domain by sampling a random point $v_{rand}$ from $\mathcal{X}$ (line \ref{AddNewSample:RandomSample}, Algorithm \ref{ExtendSearchGraph}) and
extending some parts of the graph towards $v_{rand}$. 
First, the closest node in the graph $v_{nearest}$ is found (line \ref{AddNewSample:FindNearest}, Algorithm \ref{ExtendSearchGraph}). 
Then, $v_{rand}$ is brought $\delta$-closer to $v_{nearest}$, where the values of $\delta$ is problem-specific.
The lazily-evaluated edges among $v_{new}$ and its neighbor nodes are added to $\mathcal{G}$ (line \ref{AddNewSample:AddSampleStart}-\ref{AddNewSample:AddSampleEnd}, Algorithm \ref{ExtendSearchGraph}). 
In this paper, the neighbor nodes are found by searching for nodes that are within radius $r$ of $v_{new}$, which is given in \cite{RRTstar} by 
\begin{equation}
r= \gamma_s 2\left(1+\frac{1}{d}\right)^{\frac{1}{d}}\left(\frac{\mu\left(\mathcal{X}_{\text {free }}\right)}{\zeta_d}\right)^{\frac{1}{d}}\left(\frac{\log (N)}{N}\right)^{\frac{1}{d}}
\end{equation}
where $\gamma_s$ is the tuning parameter, N=$|V|$, $d$ is the dimension of the problem, $\mu\left(\mathcal{X}_{free}\right)$ is the Lebesgue measure of the free configuration space $\mathcal{X}_{free}$, $\zeta_d$ is the volume of a unit ball in $\mathbb{R}^d$. 
The lmc- and g-value of each new node are initialized as $\infty$.
If $v_{new}$ is able to yield a solution path via one of its neighbor nodes, its lmc-value will be updated, and the node will be inserted into $\mathcal{Q}$ (line \ref{AddNewSample:UpdateNode}, Algorithm \ref{ExtendSearchGraph}).
Later, the \FuncSty{ComputeShortestPath} procedure will propagate the new information due to the extension across the graph and search for a better solution for promising nodes.

\section{Analysis}
\label{Analysis}
In this section, we provide theoretical properties concerning the probabilistic completeness and asymptotic optimality of LLPT$^*$.
First, the following theorem is given:
\begin{theorem}
\label{theorem_shortest}
Upon the termination of the $\FuncSty{ComputeShortestPath}$ procedure, 
the optimal path from $v_{start}$ to $v_{goal}$ in $\mathcal{G}$ regarding the lazy weight function $\bar{w}$ is returned, i.e., $\pi=\arg\min_{\Pi\in\mathcal{G}}\sum_{e\in\pi}\bar{w}(e)$, where $\Pi$ is the set of all finite cost paths connecting $v_{start}$ to $v_{goal}$ in $\mathcal{G}$.
Otherwise, if upon the termination there is no solution path that connects $v_{start}$ and $v_{goal}$, there is no solution regarding the current search graph.
\end{theorem}
\begin{proof}
    See Appendix \ref{proof_of_theorem_optimality}.
\end{proof}

\begin{theorem}
\label{theorem_optimality}
The LLPT$^*$ algorithm finds the optimal solution with respect to the current search graph with fully-evaluated edges if there exists one when the inner search loop (line \ref{Main:InnerLoopBegin}-\ref{Main:InnerLoopEnd}, Algorithm \ref{Main}) terminates.
\end{theorem}
\begin{proof}
The inner search loop terminates when 
(1) the $\FuncSty{ComputeShortestPath}$ procedure terminates without a solution path that connects $v_{start}$ and $v_{goal}$, and 
(2) the $\FuncSty{ComputeShortestPath}$ procedure terminates with a solution path $\pi$ which is fully evaluated and is found to be collision-free.
For the first case, as proven by Theorem \ref{theorem_shortest}, there is no solution regarding the lazy weight function $\bar{w}$. 
As $\bar{w}(e)<w(e)$ for $\forall e\in E$, the number of collision-free edges the fully-evaluated search graph has must be no more than that of the lazy-evaluated search graph. 
Hence, if the lazy-valuated search graph does not yield a solution, the fully-evaluated one can not yield one either.
For the second case, as proven by Theorem \ref{theorem_shortest}, $\pi$ is the optimal path from $v_{start}$ to $v_{goal}$ regarding the lazy weight function $\bar{w}$. 
If $\pi\in E_{eval}$, we have 
${w}(\pi)=\sum_{e\in\pi}{w}(e)=\sum_{e\in\pi}\bar{w}(e)\leq\sum_{e\in\pi^*}\bar{w}(e)
\leq\sum_{e\in\pi^*}{w}(e)$, where $\pi^*$ is the optimal solution path w.r.t. the current search graph with fully-evaluated edges. 
Therefore, $\pi$ is the optimal path w.r.t. the current search graph.
\end{proof}

\begin{theorem}[Asymptotical Optimality]
\label{AsymptoticalOptimality}
The LLPT$^*$ algorithm is asymptotically optimal as planning time $t\rightarrow\infty$, that is, LLPT$^*$ will return the optimal solution path w.r.t. the known environment and the cost function $w(\cdot)$ given infinite planning time.
\end{theorem}
\begin{proof}
The LLPT$^*$ algorithm will extend the underlying search graph infinitely by the $\FuncSty{ExtendSearchGraph}$ procedure when there is planning time remaining. 
The $\FuncSty{ExtendSearchGraph}$ procedure will create a search graph $\mathcal{G}$ of which the RRG whose edge weight is evaluated by the lazy weight function $\bar{w}$ is a subgraph, which is proven to be surely containing an optimal solution with finite cost for any non-negative cost function as long as there exists one when the number of nodes is approaching $\infty$ \cite{RRTstar}.
As proven by Theorem \ref{theorem_optimality}, the $\FuncSty{ComputeShorestPath}$ procedure will return the optimal solution w.r.t. $\mathcal{G}$ with fully-evaluated edges, or equivalently, the RRG with fully-evaluated edges. 
Therefore, LLPT$^*$ will eventually converge to the optimal solution as long as there exists one.
\end{proof}

\section{Simulation and Experiment Results}
\label{Simulation}

\begin{figure*}[htbp]
\centering
  \subfloat[$\mathbb{SE}(3)$ Cubicles]{\includegraphics[width = 0.26\textwidth, , height=0.215\textwidth]{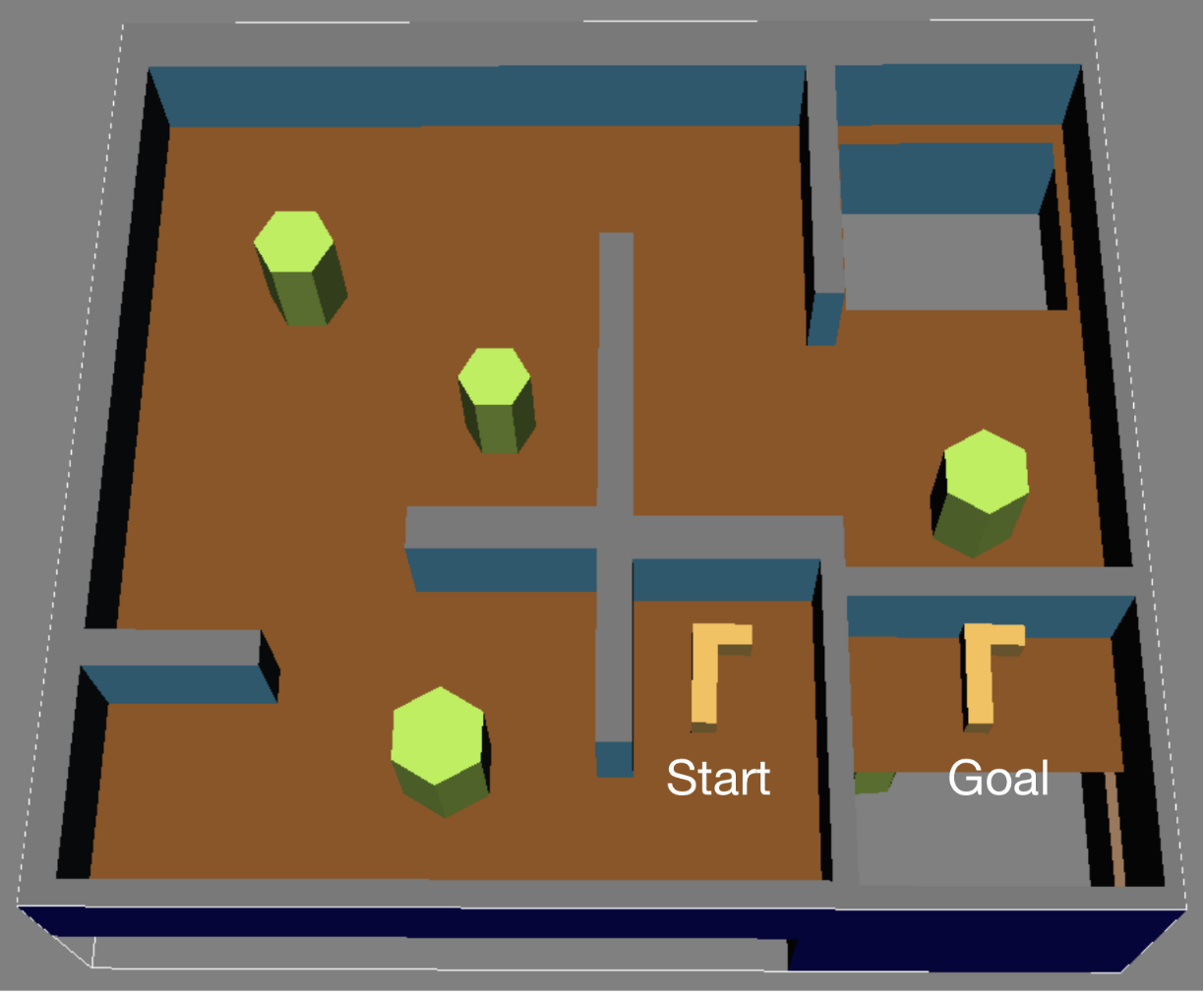}}\hfil
  \subfloat[$\mathbb{SE}(3)$ Piano Mover]{\includegraphics[width = 0.278\textwidth, height=0.215\textwidth]{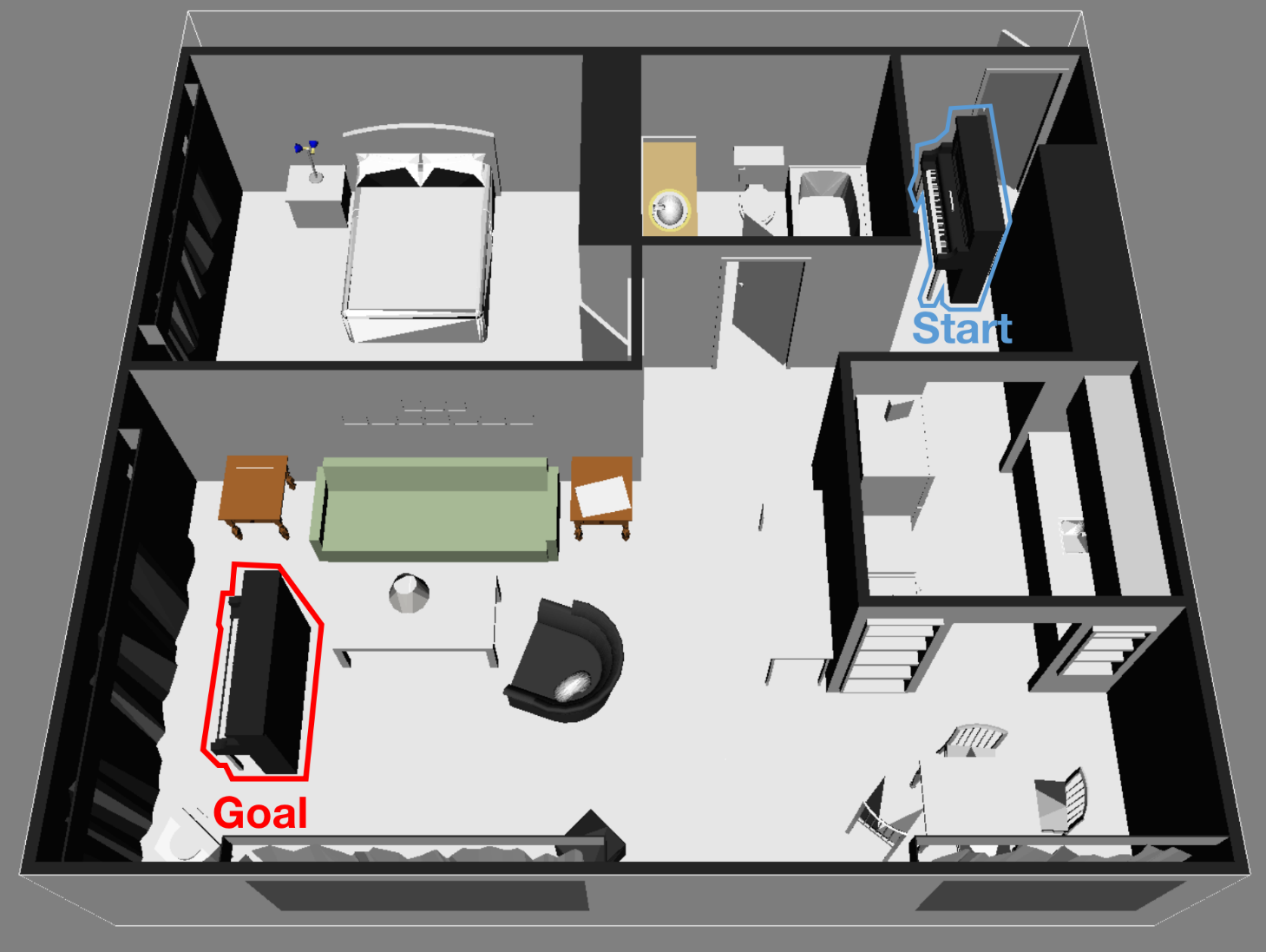}}\hfil
  \subfloat[$\mathbb{R}^7$ Emika Franka Robot Arm]
  {\includegraphics[width = 0.41\textwidth]{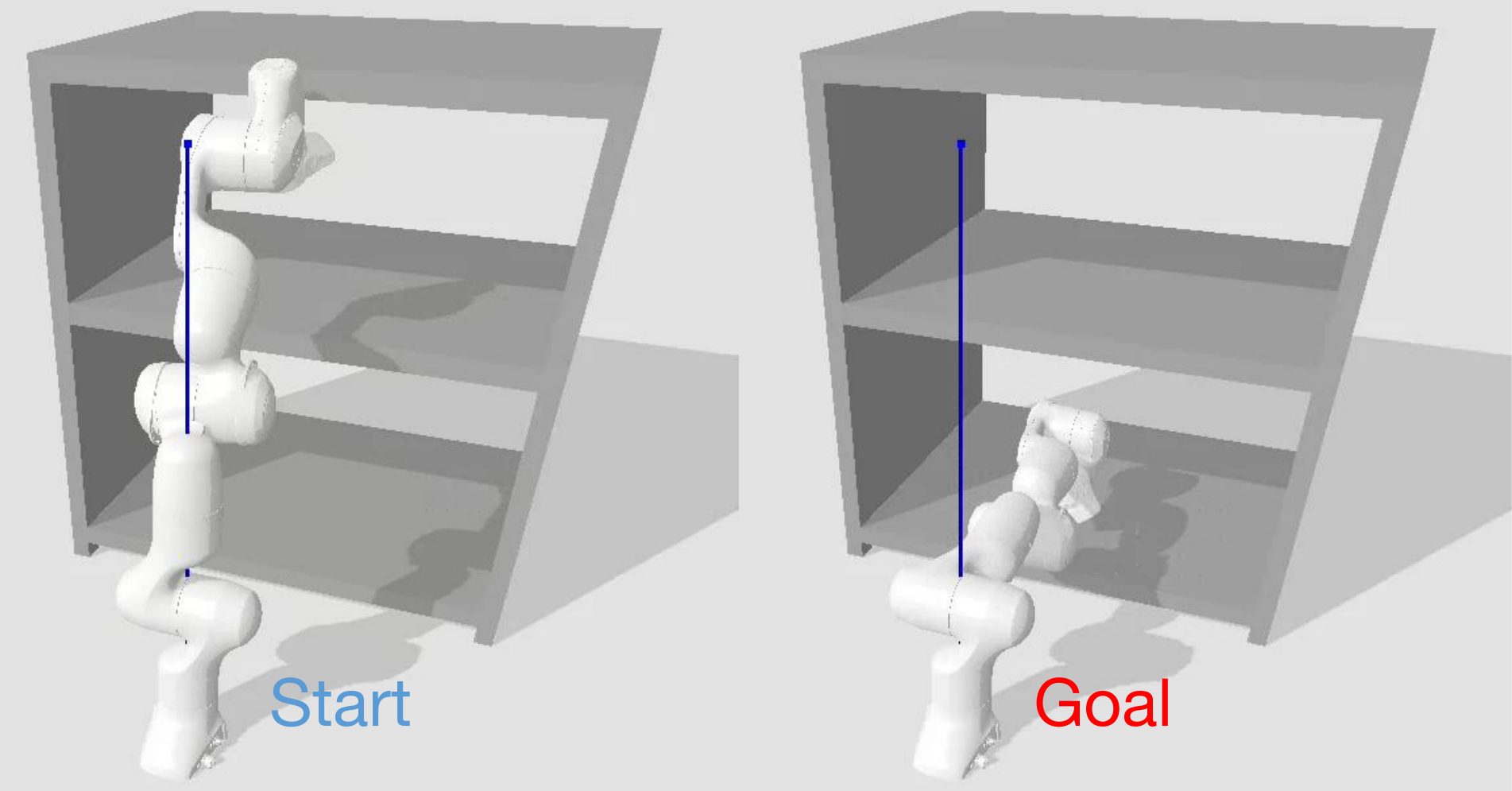}}
\caption{Depictions of rigid-body motion planning problems from the OMPL.app and PyBullet.}
\label{fig:StaticEnv}
\end{figure*}

\begin{figure*}[h]
\centering
  \subfloat[Cubicles, $\mathbb{SE}(3)$]{\includegraphics[width = 0.33\textwidth, height = 0.26\textwidth]{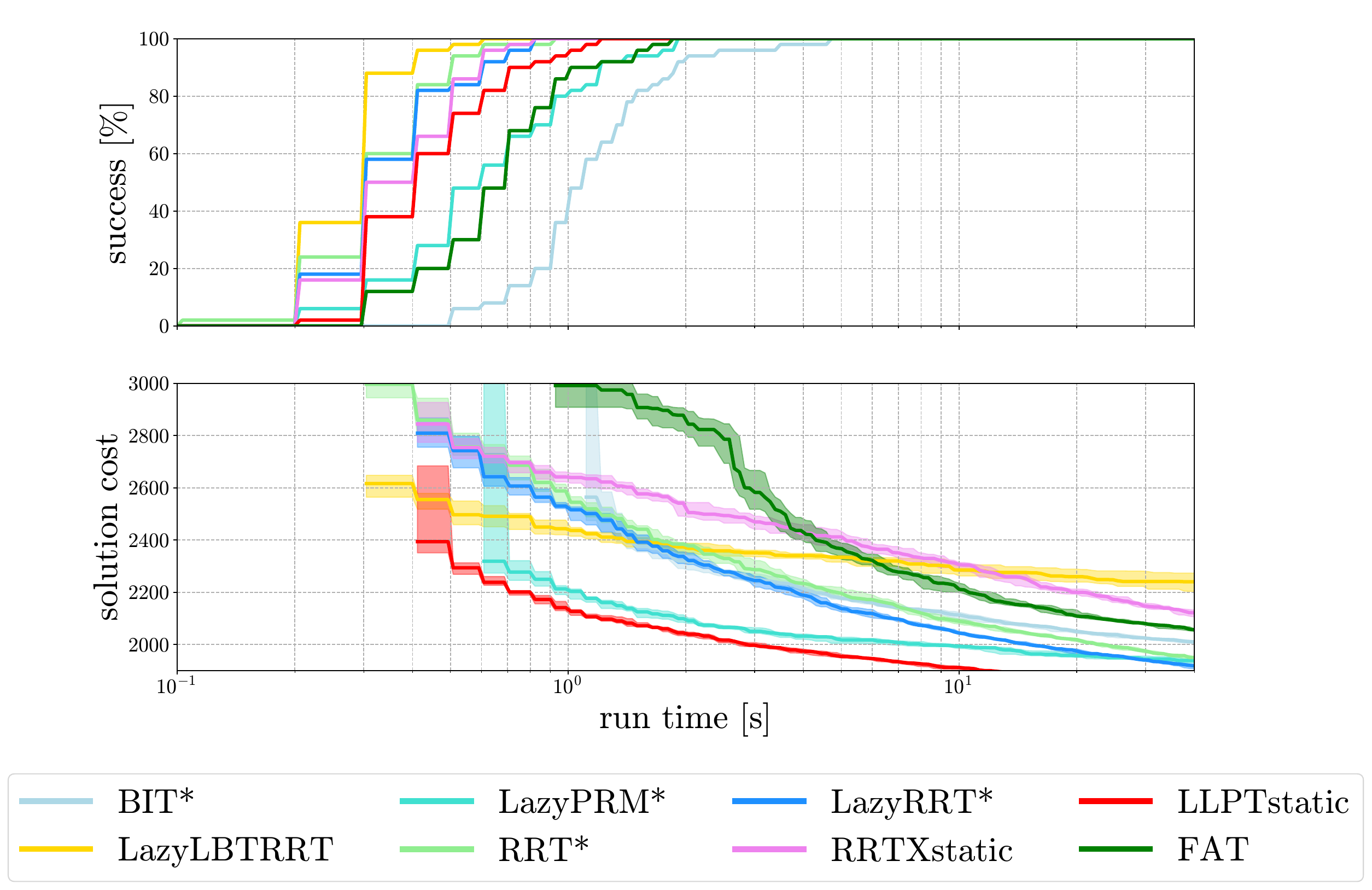}}\hfill
  \subfloat[Piano Mover, $\mathbb{SE}(3)$]{\includegraphics[width = 0.33\textwidth,  height = 0.26\textwidth]{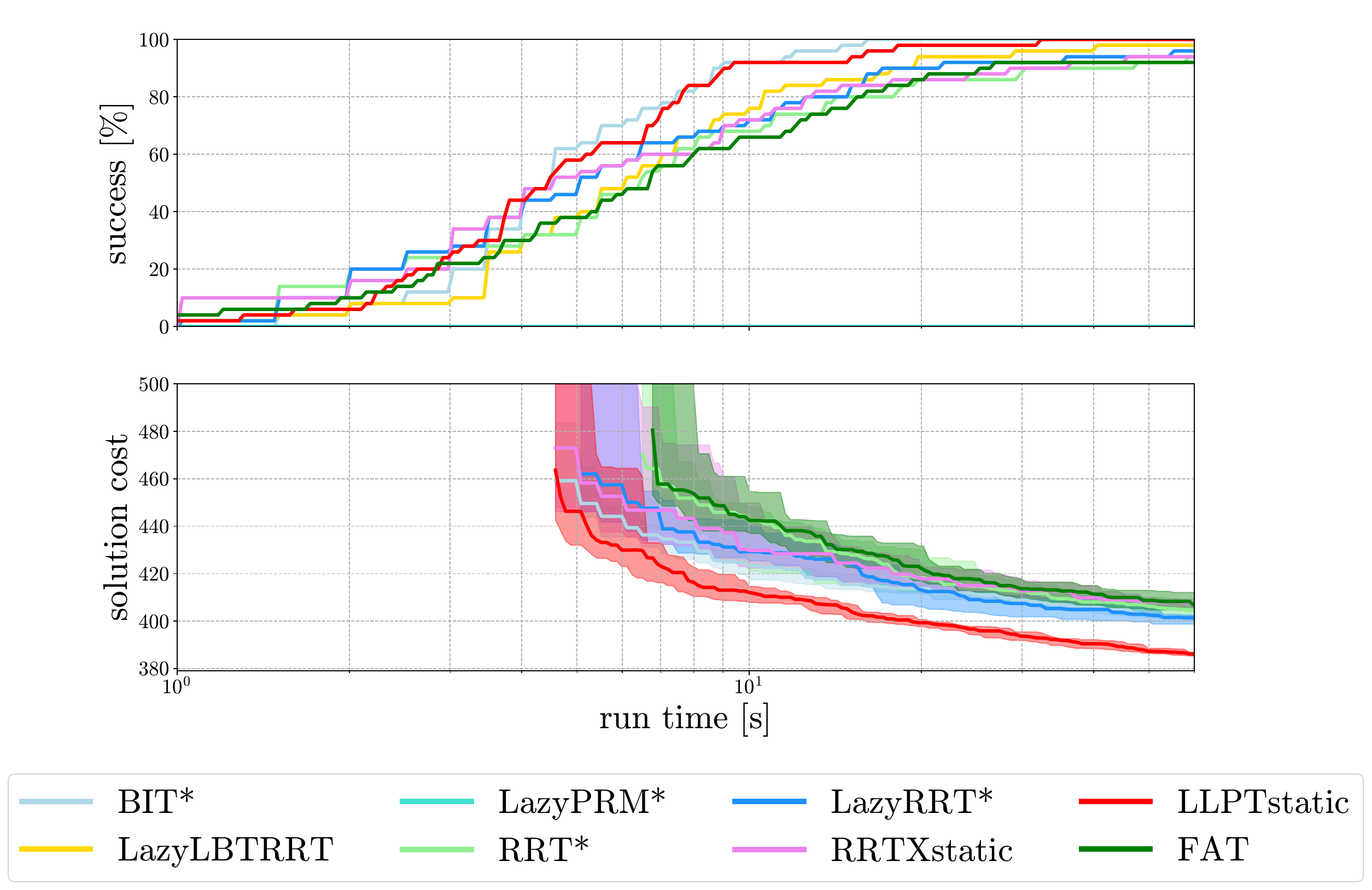}}\hfill
  \subfloat[Manipulation, $\mathbb{R}^{7}$]{\includegraphics[width = 0.33\textwidth,  height = 0.26\textwidth]{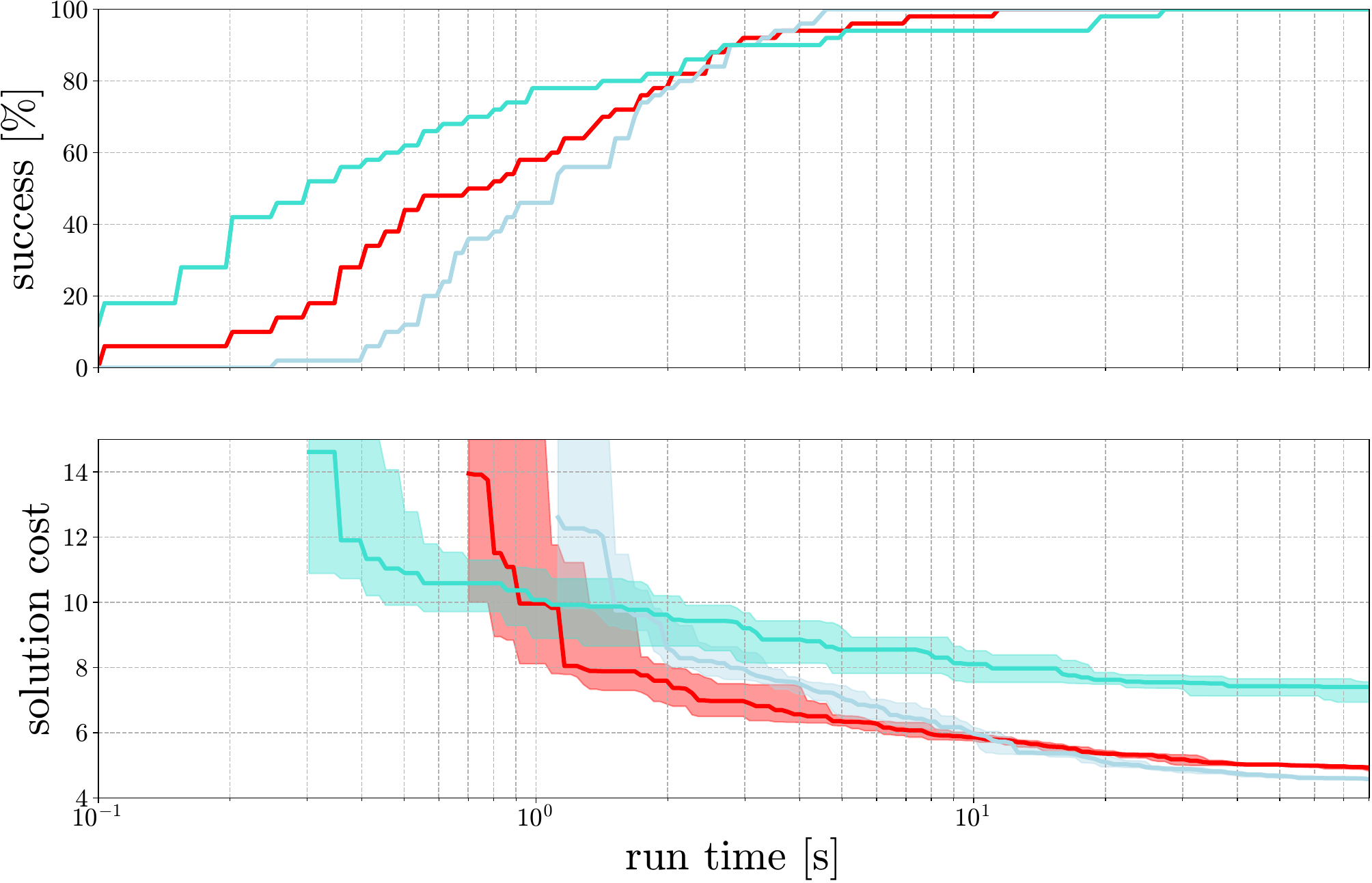}}

  \subfloat{\includegraphics[width = 0.6\textwidth]{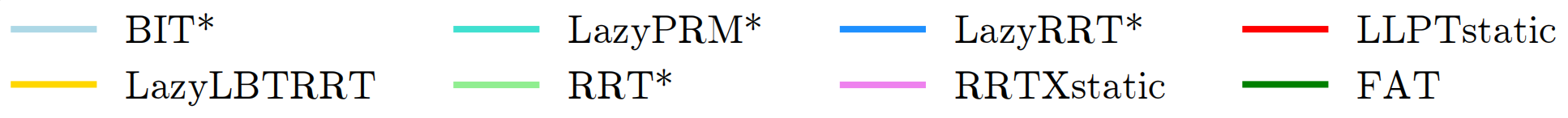}}
\caption{Success rate and medium solution cost versus time for the static path planning problems illustrated in Fig.\ref{fig:StaticEnv}. Each planner was tested with 50 different pseudo-random seeds. 
The median values are plotted with error bars denoting a non-parametric 95$\%$ confidence interval on the median.} 
\label{fig:Simulation_Static}
\end{figure*}

In this section, we present simulation results comparing the performance of LLPT$^*$ vs. other sampling-based planning algorithms with or without lazy search in solving static or dynamic path planning problems \footnote{All simulations were run on Ubuntu 22.04 on an Intel® Core™ Ultra processor (Series 1) CPU machine with 32 GB of RAM.}.
At the end of this section, we show the results of the proposed algorithm in real-world experiments.

\subsection{Baseline Algorithms}
We compare LLPT$^*$ against RRT$^*$\cite{RRTstar}, RRT$^*$ with lazy collision check (denoted as LazyRRT$^*$), LazyPRM$^*$\cite{LazyRRT}, BIT$^*$\cite{BIT}, RRT$^X$\cite{RRTX}, DRRT\cite{DRRT}, LazyLBTRRT\cite{LBTRRT}, and FAT$^*$\cite{FAT}.
In all the plots, We use 'static' as a suffix to distinguish a lifelong planner from its static version \footnote{The static version can be regarded as an instance of the lifelong version with infinite replanning intervals.}. 
To ensure a fair comparison, each planning algorithm was tested using the Open Motion Planning Library (OMPL) v1.6.0\cite{OMPL}.
For the benchmarking algorithms, we used the default OMPL settings.
We ensured that LLPT$^*$ used the same tuning parameters and configurations as the benchmarking algorithms whenever possible.
The Euclidean norm is used as the weight and the heuristic functions for all problems.
The edge resolution is set as 0.01 for collision check.
We use the r-disk RRG$^*$ as the underlying search graph for all RRT$^*$-like planners. 
The evaluation parameter $\alpha$ of LLPT$^*$ is set as 100 for all cases.

\begin{figure*}[t]
\centering
  \subfloat[$\mathbb{SE}(3)$ Cubicles, Scene 1]{\includegraphics[width = 0.26\textwidth, , height=0.215\textwidth]{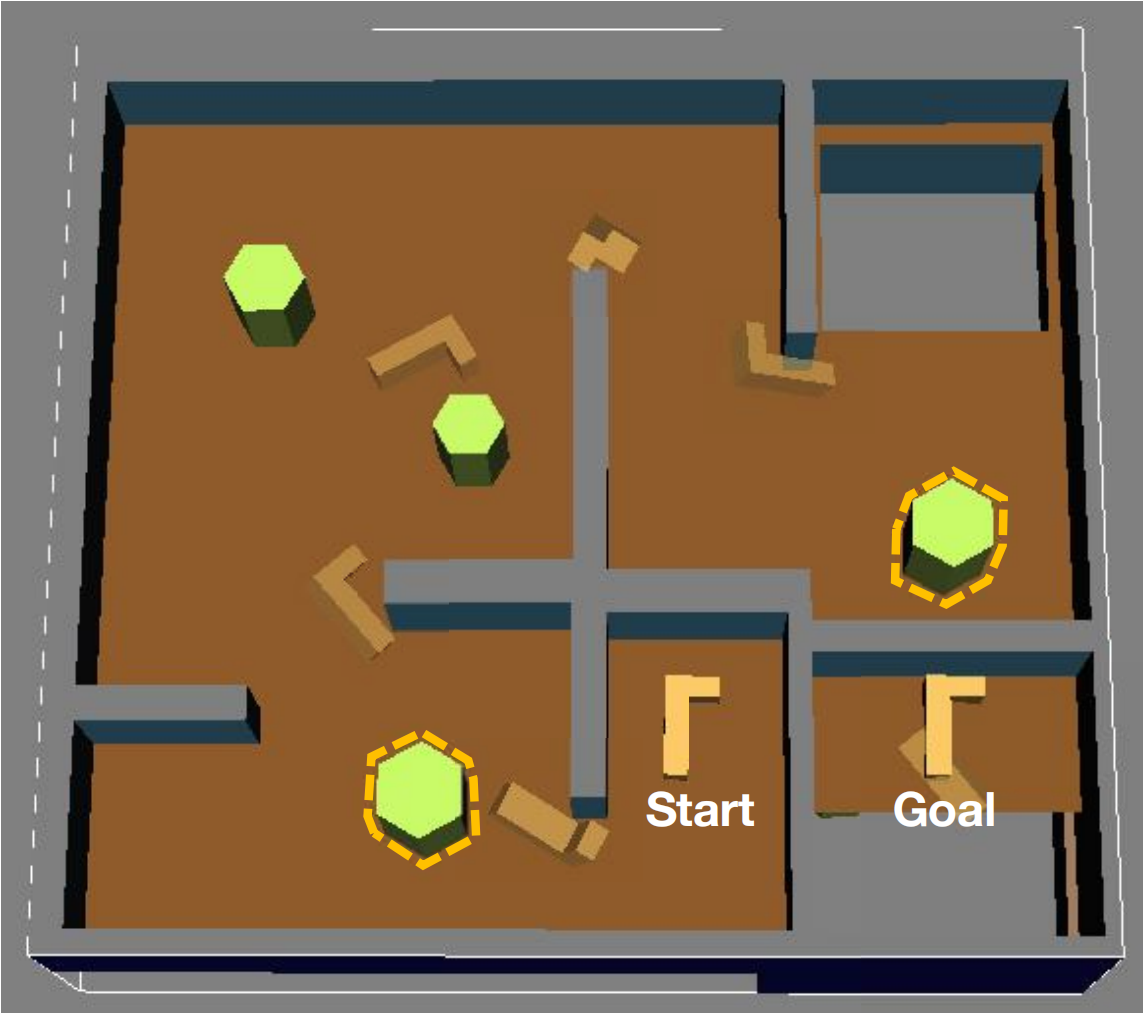}}\hfil
  \subfloat[$\mathbb{SE}(3)$ Cubicles, Scene 2]{\includegraphics[width = 0.26\textwidth, , height=0.215\textwidth]{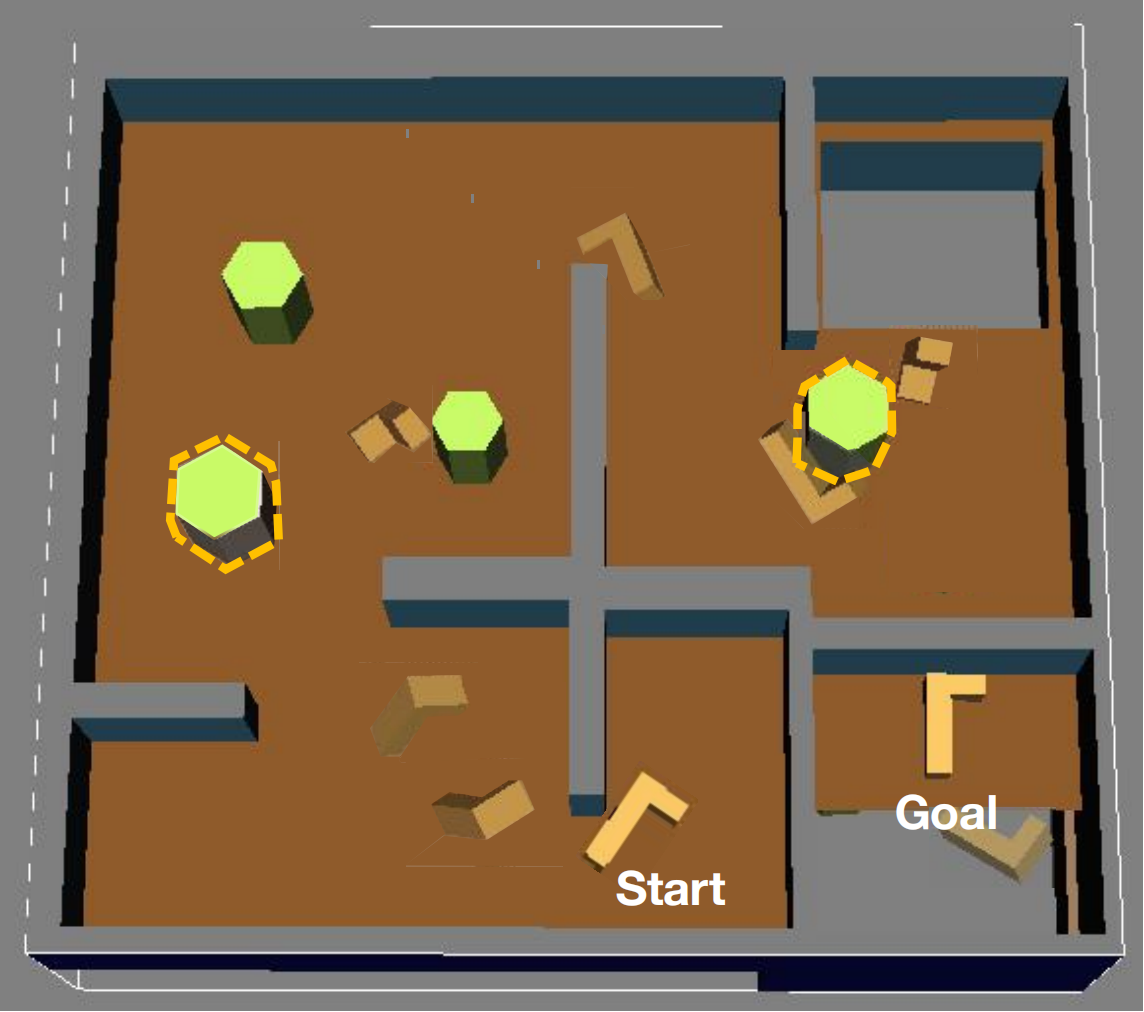}}\hfil
  \subfloat[$\mathbb{SE}(3)$ Cubicles, Scene 3]{\includegraphics[width = 0.26\textwidth, , height=0.215\textwidth]{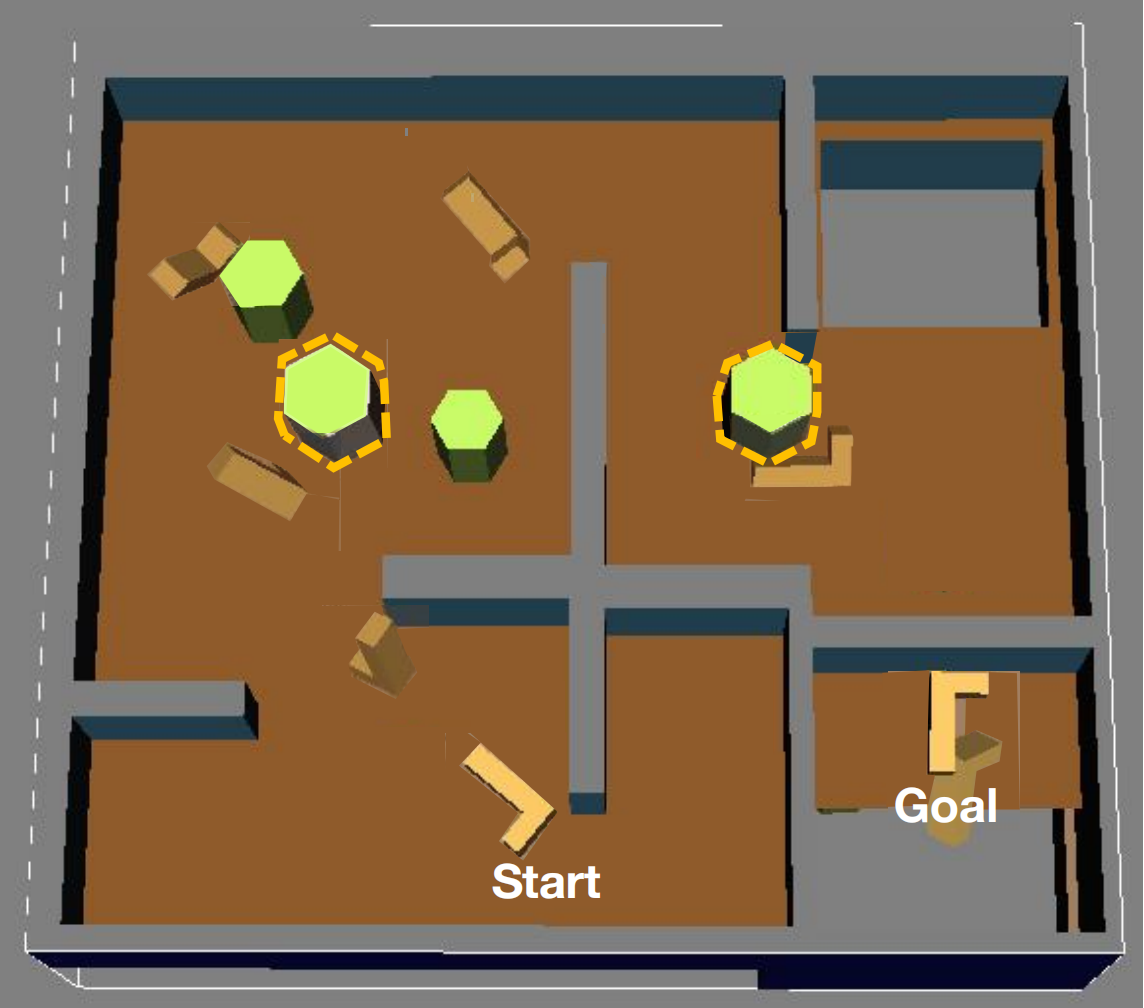}}

  \subfloat[$\mathbb{SE}(3)$ Piano Movers, Scene 1]{\includegraphics[width = 0.26\textwidth, , height=0.215\textwidth]{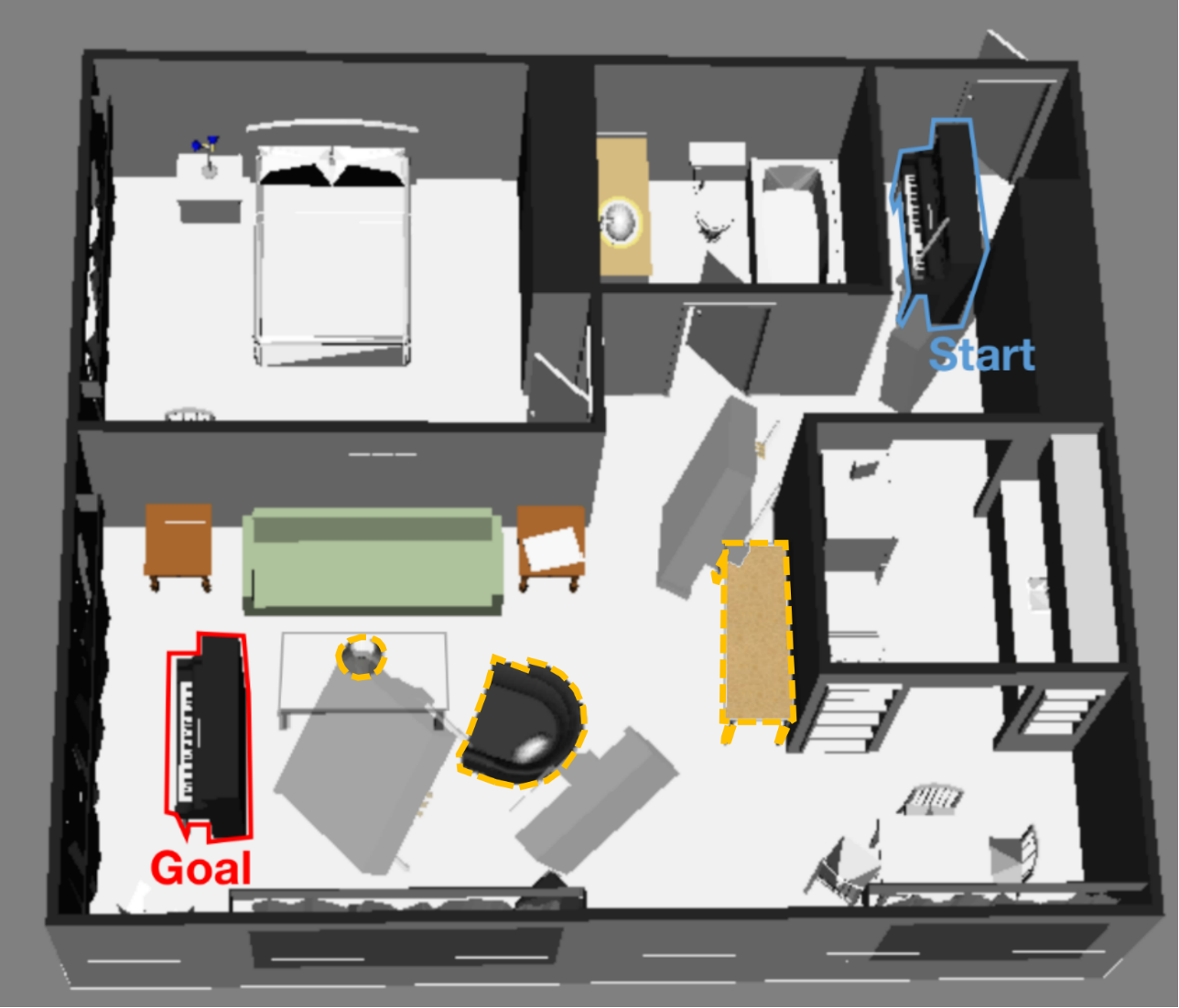}}\hfil
  \subfloat[$\mathbb{SE}(3)$ Piano Movers, Scene 2. 
  The lamp and the sofa have been moved.
  ]{\includegraphics[width = 0.26\textwidth, , height=0.215\textwidth]{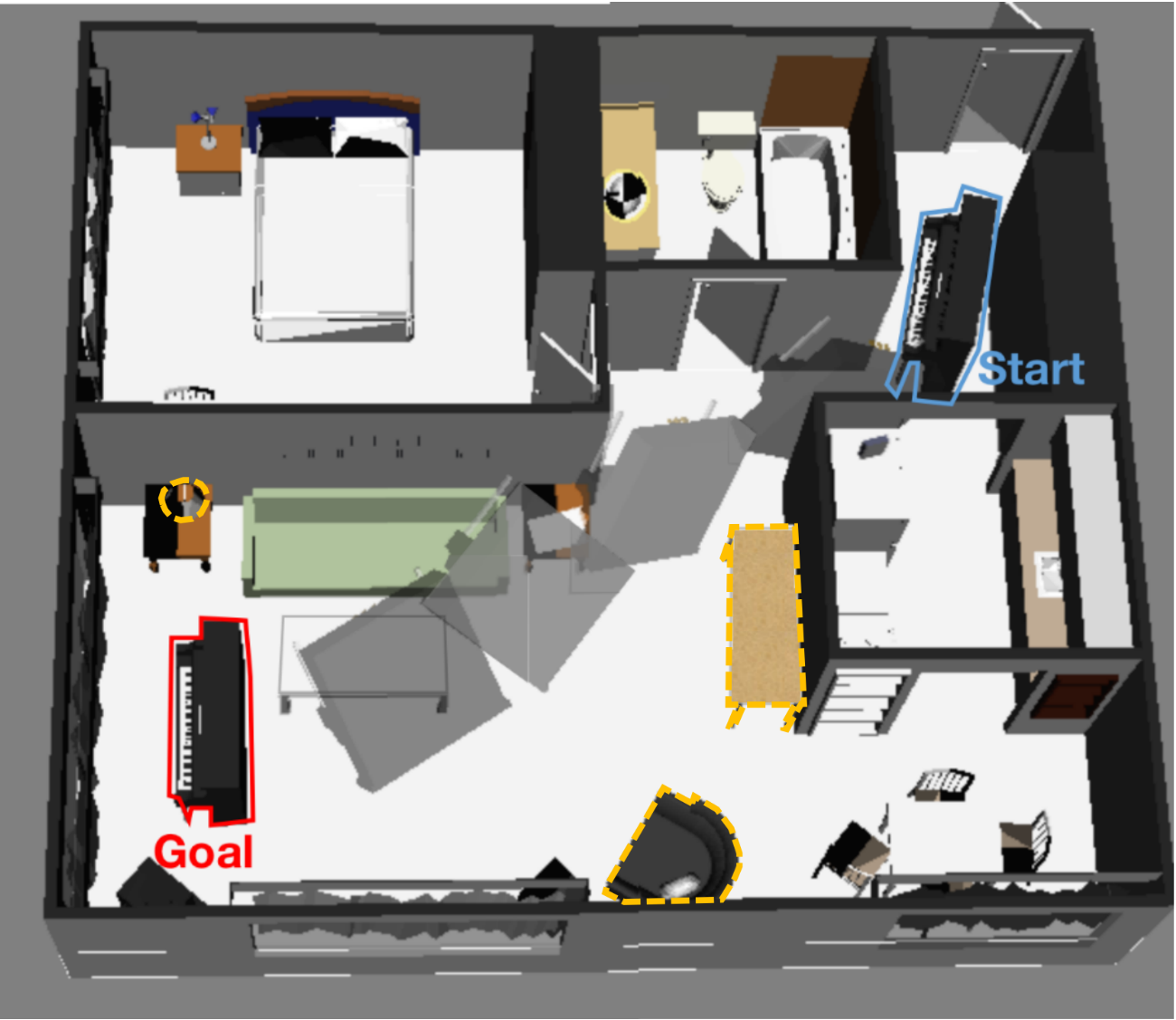}}\hfil
  \subfloat[$\mathbb{SE}(3)$ Piano Movers, Scene 3. 
  The wooden desk has been moved.
  ]{\includegraphics[width = 0.26\textwidth, , height=0.215\textwidth]{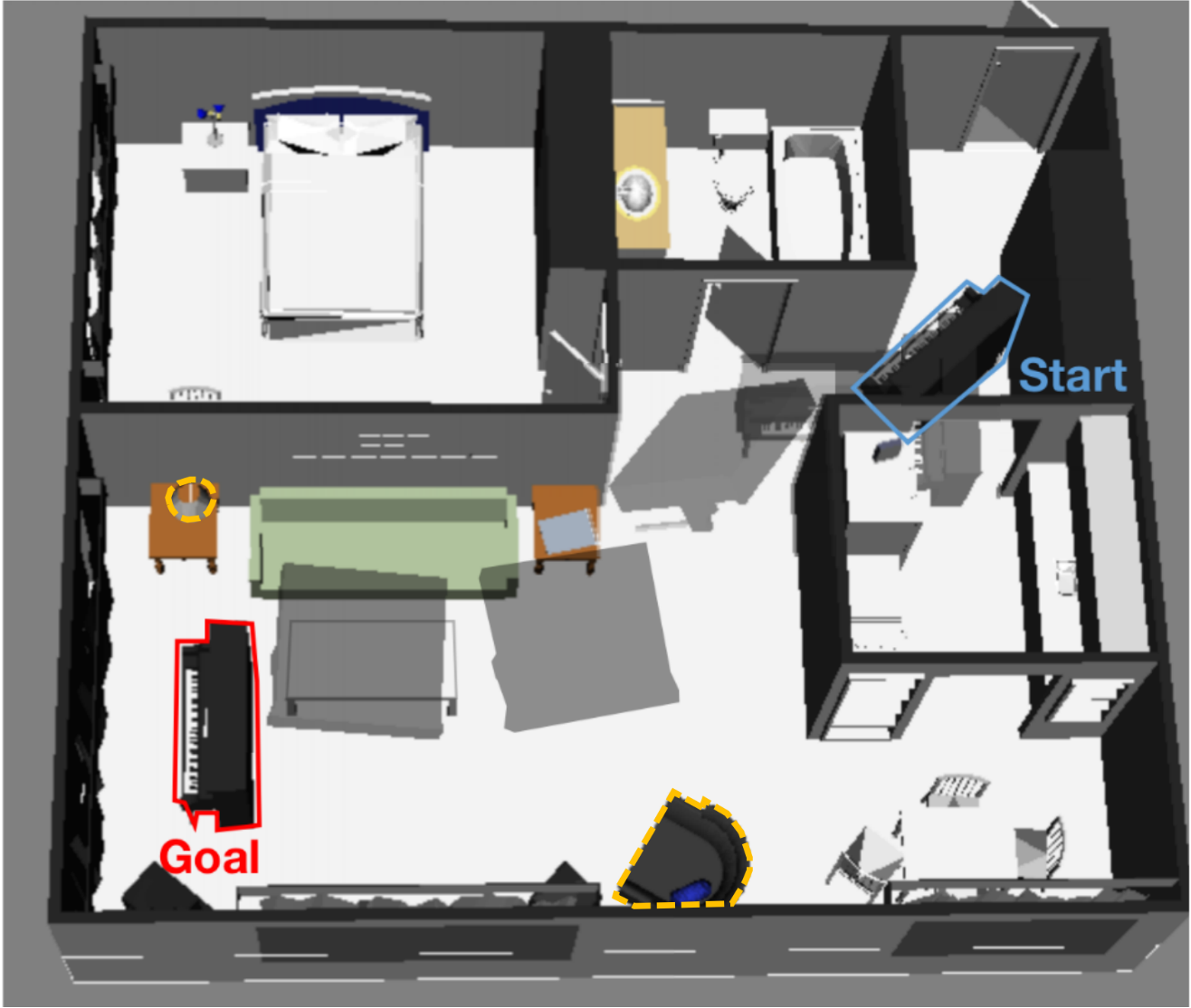}}

  \subfloat[$\mathbb{R}^7$ Manipulation, Scene 1]{\includegraphics[width = 0.26\textwidth, , height=0.25\textwidth]{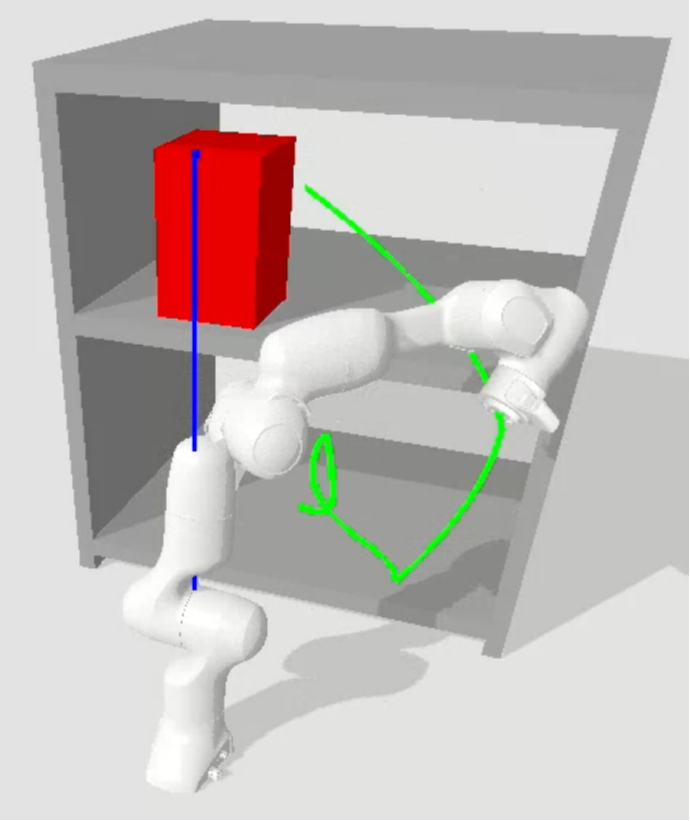}}\hfil
  \subfloat[$\mathbb{R}^7$ Manipulation, Scene 2]{\includegraphics[width = 0.26\textwidth, , height=0.25\textwidth]{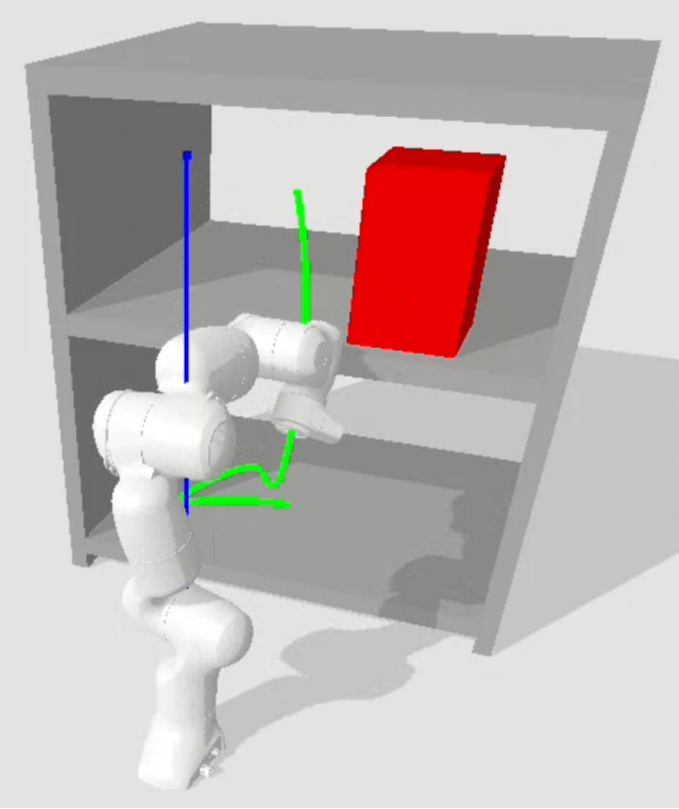}}\hfil
  \subfloat[$\mathbb{R}^7$ Manipulation, Scene 3]{\includegraphics[width = 0.26\textwidth, , height=0.25\textwidth]{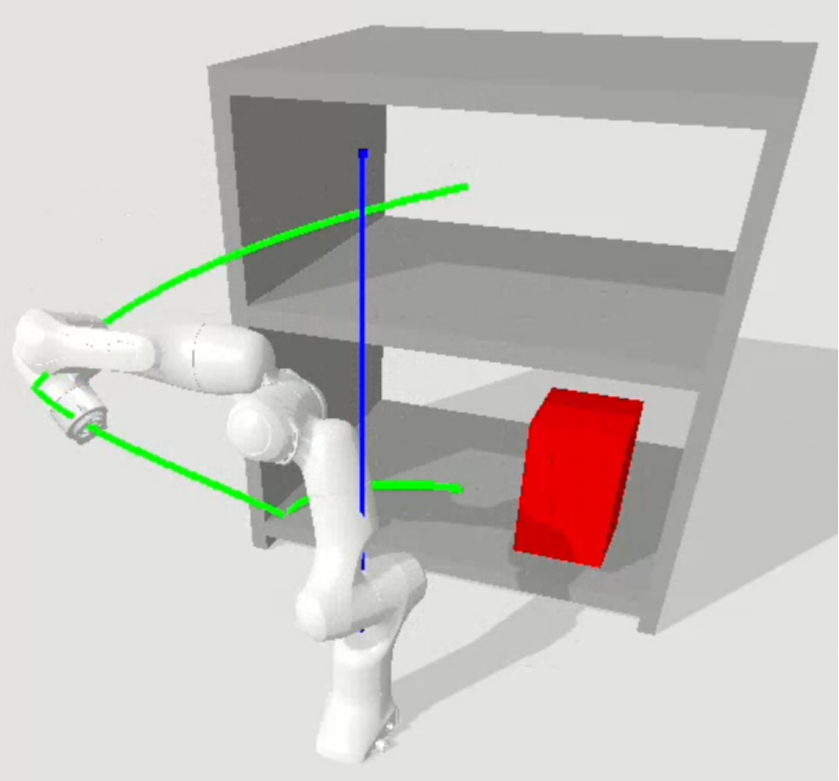}}
\caption{Depictions of rigid-body motion planning problems in dynamic environments. 
The dynamic obstacles are highlighted in orange.
The shortest collision-free paths for different scenes are displayed in the figures.}
\label{fig:DynamicEnv}
\end{figure*}

\subsection{Static Motion Planning Problems}
We considered the following problems from the OMPL's test suite:
\begin{itemize}
    \item the cubicles problem in $\mathbb{SE}(3)$ as shown in Fig.\ref{fig:StaticEnv}(a);
    \item the piano movers' problem in $\mathbb{SE}(3)$ as shown in Fig.\ref{fig:StaticEnv}(c).
\end{itemize}
Besides, we also consider the following problem to investigate the planners' performance in high-dimensional configuration spaces:
\begin{itemize}
    \item the 7-DoF manipulation planning problem on the Franka Emika Panda robot as shown in Fig.\ref{fig:StaticEnv}(d), where the Franka Emika Panda robot simulates picking/placing objects between the layers of a bookshelf. 
\end{itemize}

The simulation results are displayed in Fig.\ref{fig:Simulation_Static}, showcasing success rate and medium solution cost plotted against run time.  
Several motion planners are not shown due to their performance.
At the beginning of the planning, the success rate of LLPT$^*$ is lower than that of some other planners since it uses more computational time to find an initial solution.
However, as planning proceeds, LLPT$^*$ soon emerges as the top performer among all algorithms in all four scenarios, excelling in success rate and convergence rate towards the optimal solution. 
These results demonstrate the efficacy of lazily delaying edge evaluations in solving high-dimensional optimal motion planning problems.

\subsection{Dynamic Motion Planning Problems}
\label{Simulation:DMPP}
The planners were evaluated in three dynamic scenarios where both the positions of obstacles and the robot evolve over time, as illustrated in Fig.\ref{fig:DynamicEnv}. For each motion planner, we conducted three consecutive planning sessions: the first in Scene 1, the second in Scene 2, and the third in Scene 3. Planners without lifelong planning capabilities must replan from scratch each time the environment changes.
Table\ref{Table1} reports each planner’s performance in terms of success rate and average solution cost under different replanning time constraints. All results are averaged over 50 trials with different random seeds. To penalize failure cases, the solution cost for unsuccessful trials is set to a large value significantly higher than that of any successful one.

The results demonstrate that LLPT$^*$ consistently enables effective replanning across all dynamic environments. It achieves substantially better average performance than other planners under the same replanning time limits. In particular, as more replanning time is provided, LLPT$^*$ significantly outperforms others in terms of average solution cost.
Notably, the lifelong planners DRRT and RRT$^X$ perform considerably worse than LLPT$^*$ and even some static planners. DRRT’s subpar performance stems from its lack of asymptotic optimality, while RRT$^X$ suffers from high computational overhead due to its costly edge evaluation process.
As replanning time decreases, most planners—especially non-lifelong ones—experience a dramatic drop in success rate and a spike in solution cost. In contrast, LLPT$^*$ maintains a high success rate, low average solution cost, and low standard deviation, demonstrating strong robustness and efficiency in highly dynamic environments.


\begin{table*}[t]
\centering
\caption{Success rate (\%) and Average Solution Cost of Different Motion Planners under Different Replanning Time for Various Problems Shown in Fig.\ref{fig:DynamicEnv}}
\resizebox{\textwidth}{!}{
\begin{tabular}{c c| ccc | ccc | ccc}
\toprule
& \multirow{2}{*}{\makecell{Motion \\ Planner}} & \multicolumn{3}{c}{Cubicles} & \multicolumn{3}{c}{Piano Mover} & \multicolumn{3}{c}{Manipulation} \\
\cmidrule(lr){3-5} \cmidrule(lr){6-8} \cmidrule(lr){9-11}
& & T=2s & T=5s & T=10s & T=5s & T=10s & T=20s & T=2s & T=5s & T=10s \\
\midrule
\multirow{9}{*}{\rotatebox{90}{\textbf{Success Rate (\%)}}}
& RRT$^*$     & 100 & 100  & 100 & 26 & 54  & 66  &  0  &  0  &  0 \\
& BIT$^*$     & 72  & 100  & 100 & 30 & 90  & 100 & 38  & 72  & \textbf{98}  \\
& LazyPRM$^*$ & 94  & 100  & 100 & 2  & 4   & 6   & 32  & 46  & 70  \\
& LazyRRT$^*$ & 100 & 100  & 100 & 28 & 46  & 72  & 0   &  0  &  0 \\
& LazyLBTRRT  & 100 & 100  & 100 & 34 & 60  & 66  & 0   &  0  &  0 \\
& RRT$^X$     & 0   & 0    & 0   & 0  & 0   &  0  & 0   &  0  &  0 \\
& DRRT        & 88  & 100  & 100 & 20 & 32  & 36  & 66  &  \textbf{90} & 94  \\
& \textbf{LLPT$^*$}    & \textbf{100} & \textbf{100}  & \textbf{100} & \textbf{50} &  \textbf{90} & \textbf{100} & \textbf{76}  &  88 & 94  \\
\midrule
\multirow{9}{*}{\rotatebox{90}{\textbf{Solution Cost }}}
& RRT$^*$     & 1803.5±32.1 & 1738.9±22.2 & 1691.9±20.1 &  388.1±39.4 & 362.1±34.5 & 355.2±35.7 & -  & -  & -  \\
& BIT$^*$     & 1895.0±236.0& 1737.9±22.7 & 1699.4±22.1  & 385.6±43.8  & 340.0±16.3  & 329.1±6.1  & 17.1±8.3  & 11.0±6.3  & 6.9±1.3 \\
& LazyPRM$^*$ & 1766.7±98.5 &1706.1±20.7  & 1686.7±19.62 & 469.7±40.1  &  451.5±50.2  & 433.9±39.0  & 17.7±7.4  & 15.1±7.0  &  10.7±3.8 \\
& LazyRRT$^*$ & 1807.7±23.2 &1739.7±21.0  & 1692.6±21.6  & 387.3±39.9  & 363.1±31.2  &  350.7±33.2 & -  & -  & -  \\
& LazyLBTRRT  & 1797.4±33.0 &1765.6±36.9  & 1742.5±27.7  & 386.0±48.9  & 376.8±32.1  & 360.7±41.4  & -  & -  &  - \\
& RRT$^X$     & - & - & - & -  &  - & -  &  - & -  &  - \\
& DRRT        & 2051.6±193.8 & 2058.8±249.9 & 2024.0±161.2 & 461.2±61.0  & 446.4±67.3   & 442.0±70.4  & 20.5±5.8  & 19.5±6.0   & 19.4±5.4  \\
& \textbf{LLPT$^*$}    & \textbf{1643.7±18.8} & \textbf{1611.4±17.5} & \textbf{1592.5±16.5}  & \textbf{373.3±53.7}  & \textbf{331.7±21.0}  & \textbf{324.0±4.6}  & \textbf{11.3±7.5}  & \textbf{8.3±4.9}  & \textbf{6.7±2.0}  \\
\bottomrule
\end{tabular}
}
\label{Table1}
\end{table*}

\subsection{Real-world Experiments}
We carried out real-world experiments in a dynamic environment with moving pedestrians.
We integrated LLPT$^*$ with the Nav2 framework \cite{nav2} and deployed the navigation stack on board the TurtleBot 4 platform (see Fig.\ref{fig:tb4}).
While moving from the start to the goal position, the robot perceives the environment with the lidar sensor, updates a 2-D grid map of the environment, and updates the search tree using LLPT$^*$ to avoid static obstacles and pedestrians.
Then, a new trajectory is generated and passed to the controller for execution.
The whole navigation stack is executed in a real-time manner until the robot reaches the goal position. 
The experiment demonstrates the algorithm’s ability to handle static and dynamic elements in real-time,
ensuring efficient and robust navigation.
The video is available at \url{https://youtu.be/49H_wD0pA5U}.

\begin{figure}[t]
    \centering
    \includegraphics[width=0.3\textwidth]{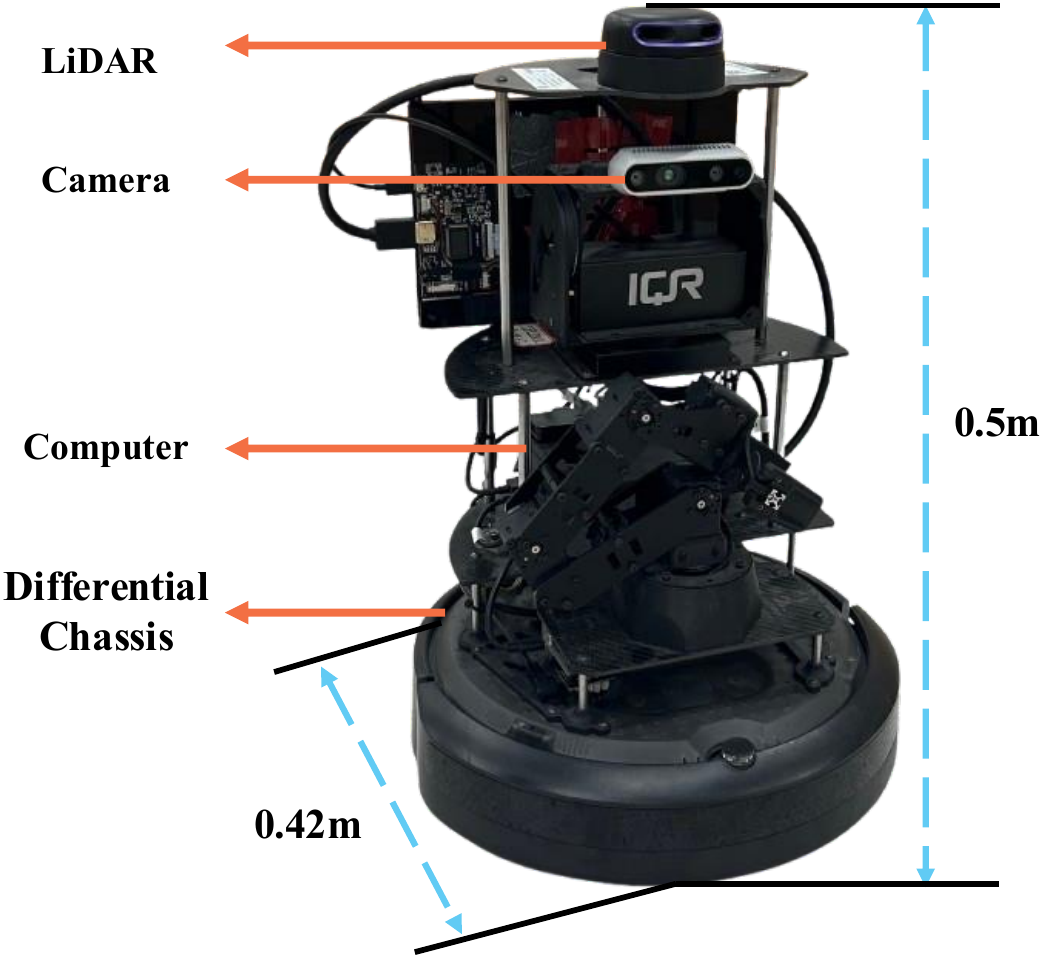}
    \caption{The TurtleBot 4 platform for real-world experiments. The robot is equipped with on an onboard computer (NUC11 with i5-1135G7 CPU).} 
    \label{fig:tb4}
\end{figure}

We present some snapshots from experiments in Fig.\ref{fig:tb4_exp_snap}.
We note that LLPT$^*$ only carries collision checking around moving pedestrians who block the optimal solution, and only a small number of rewirings are performed for each replan despite the drastic changes of the occupancy map.
Besides, LLPT$^*$ keeps densifying the search tree during the entire navigation procedure, and therefore is able to generate a trajectory with a cost close to the optimal one.

\begin{figure*}
    \centering
    \includegraphics[width=\linewidth]{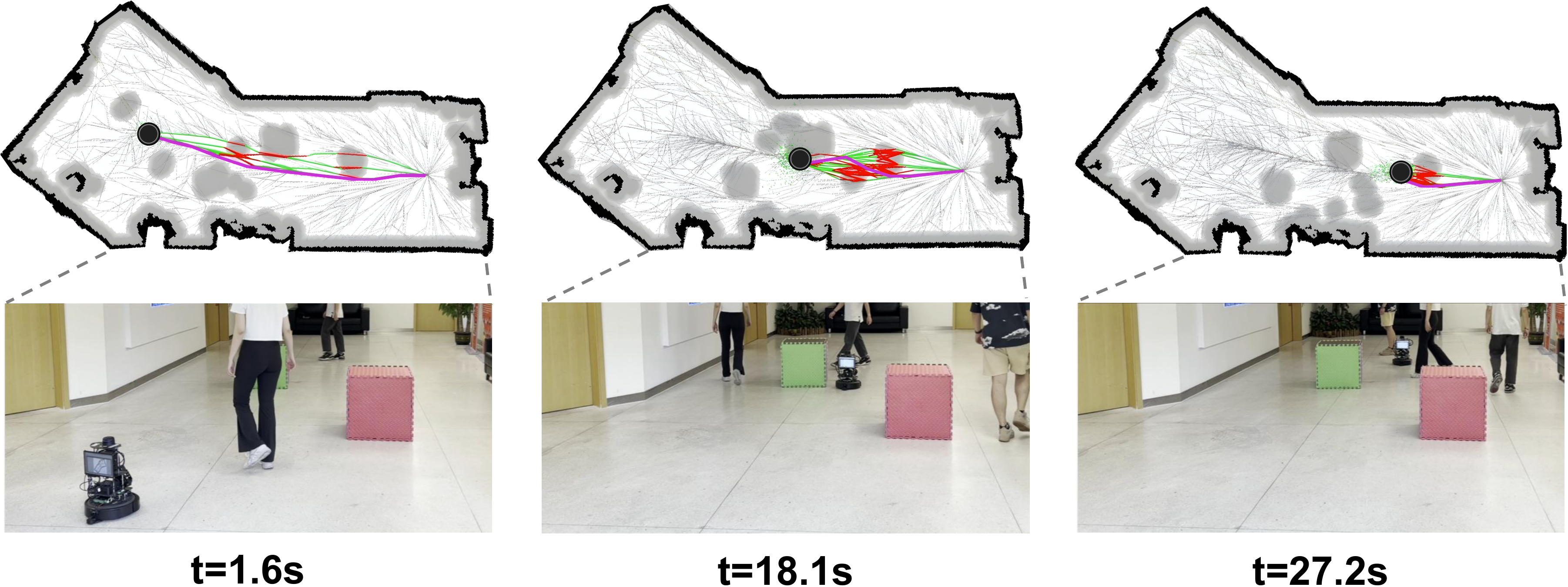}
    \caption{ Snapshots captured at 3 distinct time points from the real-world experimental process using LLPT$^*$. 
    On the top figures, the search tree of LLPT$^*$ is plotted with colorful lines representing edges with different statuses.
    The green lines are the evaluated and valid edges.
    The red lines are the evaluated and invalid edges.
    The purple lines are the solution edges.
    }
    \label{fig:tb4_exp_snap}
\end{figure*}

\section{Discussions}
\subsection{The impact of the evaluation parameter $\alpha$}
We explore the impact of the evaluation parameter $\alpha$ on LLPT$^*$ performance. 
We compared the LLPT$^*$ algorithm using various values of $\alpha$, which represents the number of edges evaluated per iteration, in solving dynamic motion planning problems described in Section \ref{Simulation:DMPP} to assess its impact on performance. 
The maximum number of nodes of the search graph was fixed for different $\alpha$ to guarantee solution consistency. 
We summarize the average numbers of total edge evaluations and vertex expansions and the average planning time in \ref{fig:discussion_alpha}. 
The number of vertex expansions refers to the number of vertices that have been updated in the \FuncSty{ComputeShortestPath} procedure.

We observed that larger values of $\alpha$ result in more edge evaluations but fewer vertex expansions. 
This is because the larger the number of edges the planner evaluates at the same time, the more knowledge it gains about the state space and, therefore, the less rewiring of the search tree.
For the cubicle problem, the replanning time remains relatively stable as $\alpha$ varies.
In the Piano Mover scenario, planning time initially decreases at $\alpha=5$ due to reduced vertex expansions, but as $\alpha$ increases further, edge evaluations begin to dominate the computation.
For the manipulation problem, edge evaluation is more expensive than vertex expansion.
Thus, the replanning time increases as $\alpha$ increases.

The simulation results suggest that the trade-off between edge evaluations and vertex expansions is crucial for determining the convergence rate to the optimal solution. 
For scenarios where edge evaluation is relatively more expensive compared with vertex expansions, such as robot arm manipulation, or motion planning for systems under differential constraints, a small value of $\alpha$, such as 1, is preferred. 
In the contrast case where the cost for edge evaluations is computationally insensitive, a larger value of $\alpha$, such as $\infty$, is a better choice since it evaluates all edges on the solution and, therefore, will minimize the number of vertex expansions and facilitate the overall replanning progress.

\begin{figure}[htbp]
\vspace{-0.1cm} 
\centering
  \subfloat[Cubicles]{\includegraphics[width = 0.16\textwidth,height=0.18\textwidth]
  {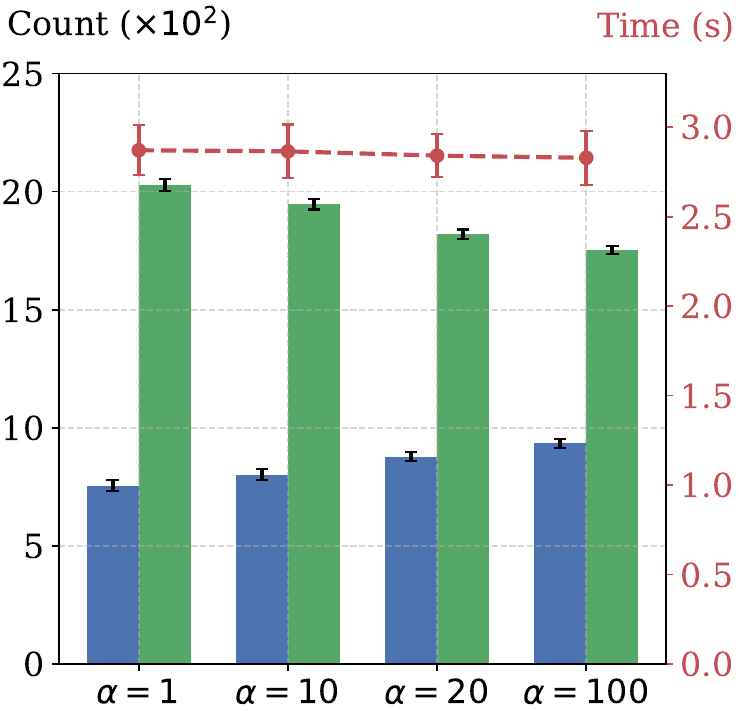}}\hfill
  \subfloat[Piano Movers]{\includegraphics[width = 0.16\textwidth,height=0.18\textwidth]{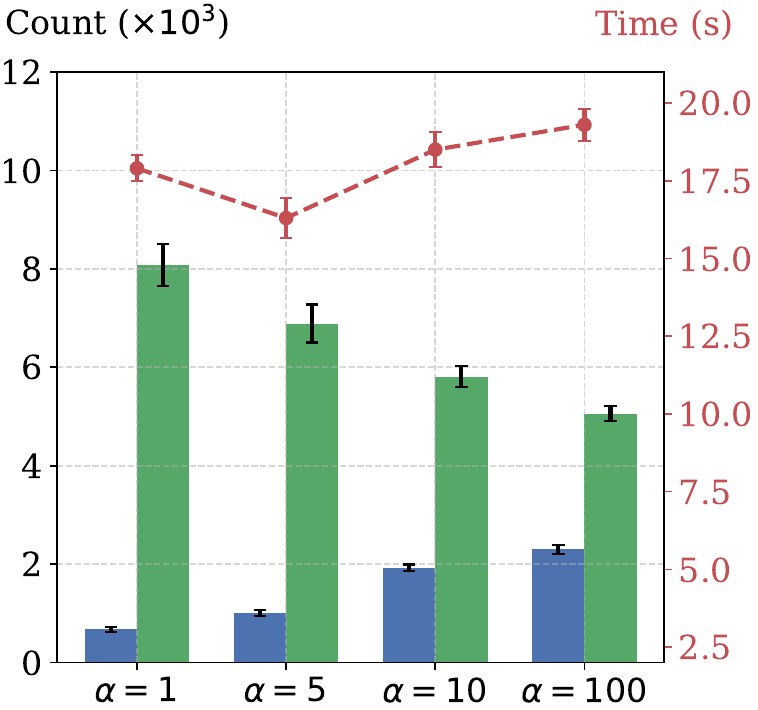}}\hfill
  \subfloat[Manipulation]{\includegraphics[width = 0.16\textwidth,height=0.18\textwidth]{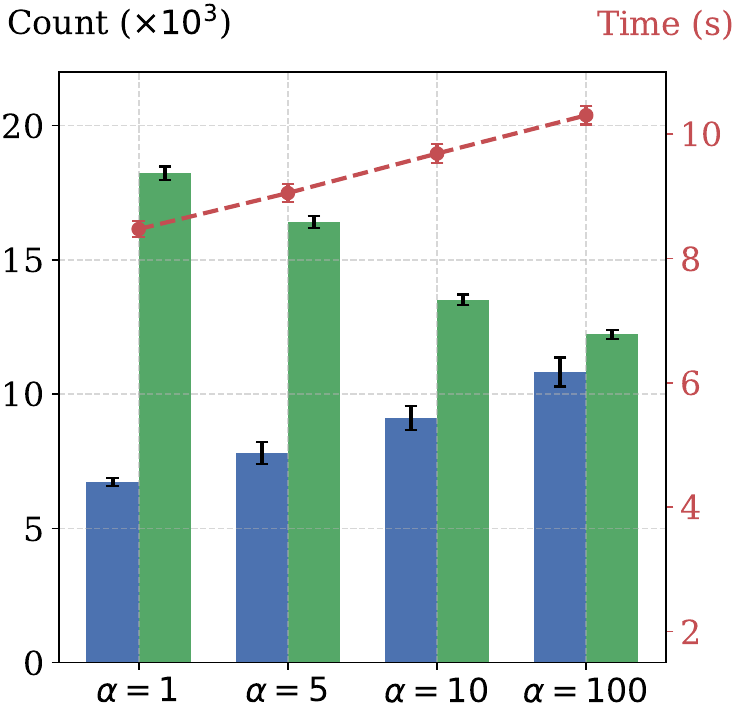}}

  \subfloat{\includegraphics[width=0.4\textwidth]{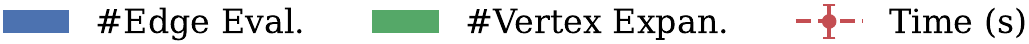}}
\caption{Simulation results of LLPT$^*$ with varying evaluation parameter.}
\label{fig:discussion_alpha}
\end{figure}


\subsection{Performance in Highly-dynamic Scenarios}
We compare the performance of LLPT$^*$ with baseline algorithms across various highly dynamic environments populated by numerous continuously moving obstacles (see Fig.\ref{fig:highly_dynamic_simulation_env}). 
Planning results are presented at \url{https://youtu.be/49H_wD0pA5U}. The results demonstrate that LLPT$^*$ can replan and compute optimal solutions at high frequency in response to new obstacle configurations and updated robot positions. 
This allows LLPT$^*$ to avoid dynamic obstacles in real time, even without access to their future trajectories.
During each replanning step, LLPT$^*$ evaluates and rewires only a small portion of the search tree near the current solution path. 
This localized update strategy enables significantly higher robot travel speeds without collisions, outperforming baseline methods.

\begin{figure}[htbp]
\vspace{-0.1cm} 
\centering
  \subfloat[]{\includegraphics[width = 0.145\textwidth, height=0.145\textwidth]{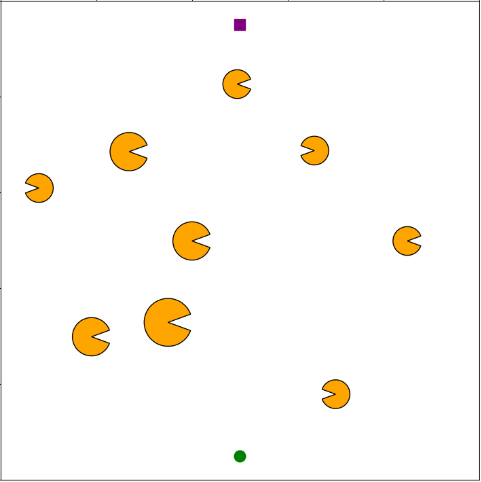}}\hfill
  \subfloat[]{\includegraphics[width = 0.145\textwidth, height=0.145\textwidth]{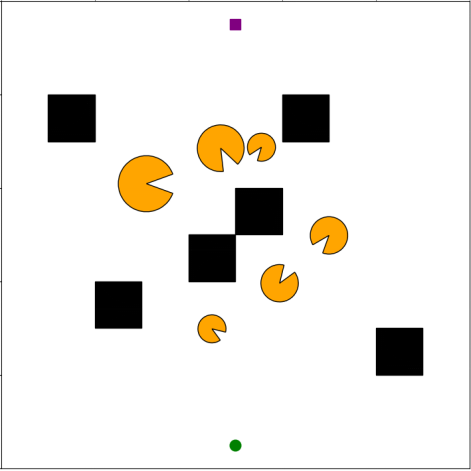}}\hfill
  \subfloat[]{\includegraphics[width = 0.175\textwidth, height=0.145\textwidth]{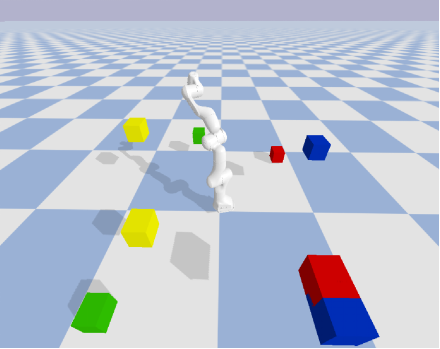}}
\caption{Simulation environments with continuously-moving obstacles.}
\label{fig:highly_dynamic_simulation_env}
\end{figure}

\subsection{Future Extensions: Improve Convergence Rate by Adaptive Risk-free Region Sampling}
Many sampling-based planners exploit the structure of the environment or past planning experience and place samples in promising areas to reduce the time and computational cost of unnecessary
exploration.
Some of them analyze the geometric structure of the state space and generate random samples from the key area, which lies between the goal and the start region \cite{InformedRRTstar} or narrow passages which are difficult to cover by uniform sampling \cite{MediaAxis, WorkspaceImportance}.
Recently, learning-based biased sampling is widely explored, which generates the optimal sampling-distribution by learning the probabilistic distribution of free state from past planning experience \cite{Learning, CVAE, LEGO, LocalPremitive, SPARKandFLAME, NeuralRRT, NAMR_RRT}.

Like these works, the convergence rate of LLPT$^*$
will be further improved if combined with the biased sampling strategy.
Nonetheless, most existing works target creating biased samples in static environments.
In dynamic environments, the distribution of dynamic obstacles is evolving, and therefore, the optimal sampling distribution should take the future trajectory of dynamic obstacles into consideration to reduce the risk of collision.
In this way, an interesting research direction 
is combining LLPT$^*$ with adaptive heuristic region sampling that reduces both the collision risk and the overall planning time.
Furthermore, as LLPT$^*$ reuses the search tree throughout the entire navigation process, some parts of the tree should be pruned since they no longer lie in the promising area as the environment evolves. 
Heuristic-based approaches can be developed to identify and prune the useless tree branches to reduce storage and computation costs.


\section{Conclusions}
\label{Conclusion}
We introduce LLPT$^*$, an asymptotically optimal lazy lifelong sampling-based path planning algorithm that combines the replanning efficiency of lifelong planning algorithms, the search efficiency of heuristic-guided informed algorithms, and the evaluation efficiency of lazy search algorithms. 
The simulation and the experiment results demonstrate that LLPT$^*$ surpasses several sampling-based planning algorithms to solve static and dynamic motion planning problems in varying dimensions, highlighting its potential applications in high-DoF robotics systems operating in highly dynamic environments.

\appendix
\subsection{Proof of Theorem \ref{theorem_shortest}}
\label{proof_of_theorem_optimality}
Let $v$ be the node selected for expansion, or equivalently, the node in $\mathcal{Q}$ with the highest priority 
at the $i^{th}$ iteration of the $\FuncSty{ComputeShortestPath}$ procedure. 
As the priority of the queue is equivalent to the priority ordered by the key $k(v)=[k_1(v), k_2(v)]$, where $k_1(v)=\min(lmc(v),g(v))+h(v)$ and $k_2(v)=\min(lmc(v),g(v))$, we use this equation to represent the key of nodes in this proof for clarity.
Throughout the section, $k^{i-1}(v)=[min(lmc^{i-1}(v),g^{i-1}(v))+h(v,v_{start}), min(lmc^{i-1}(v),g^{i-1}(v))]$ denotes the key of $v$ before expansion, that is, right before the line \ref{ReduceInconssistency:UnderconsistentStart} of Algorithm \ref{ComputeShortestPath} is executed, and $k^i(v)=[min(lmc^{i}(v),g^{i}(v))+h(v,v_{start}), min(lmc^{i}(v),g^{i}(v))]$ denotes the key after expansion, that is, right after the line \ref{ReduceInconsistency:Consistent} of Algorithm \ref{ComputeShortestPath} is executed. 
We have the following lemmas.
\begin{lemma}
\label{lemma_min_key}
The key of $v$ before expansion is no larger than that of other inconsistent nodes.
\end{lemma}
\begin{lemma}
\label{lemma_overconsistent}
If $v$ is overconsistent before expansion, i.e., lmc$^{i-1}$(v)$<$g$^{i-1}$(v), then its key after expansion, $k^i(v)$, will be equal to $k^{i-1}(v)$.
\end{lemma}
\begin{proof}
For an overconsistent node, $lmc^{i-1}(v) < g^{i-1}(v)$ and $k^{i-1}(v) = [lmc^{i-1}(v)+h(v,v_{start}), lmc^{i-1}(v)]$.
The $\FuncSty{ComputeShortestPath}$ procedure will not change its lmc-value as the conditions of line \ref{ReduceInconssistency:UnderconsistentStart} is not satisfied. 
Since $g^i(v) = lmc^{i-1}(v)<\infty$ after line \ref{ReduceInconsistency:Consistent} is executed, $k^{i}(v)=k^{i-1}(v)$ must be true. 
\end{proof}

\begin{lemma}
\label{lemma_underconsistent}
If $v$ is underconsistent before expansion, i.e., lmc$^{i-1}$(v)$>$g$^{i-1}$(v) or lmc$^{i-1}$(v)$=$g$^{i-1}$(v)$=\infty$, lmc$^{i}$(v)$\geq$g$^{i-1}$(v) must be true.
\end{lemma}
\begin{proof}
Assuming that $lmc^{i}(v)<g^{i-1}(v)$, $v$ must find a parent node $u$ such that $lmc^{i}(v)=lmc^{i-1}(u)+\bar{w}(v,u)$.
If such $u$ exists, its key before the i$^{th}$ expansion satisfies 
\begin{equation*}
    \begin{split}
        k^{i-1}(u)\dot{\leq}&[lmc^{i-1}(u)+h(u,v_{start}),lmc^{i-1}(u)]\\
        =&[lmc^{i}(v)-\bar{w}(v, u)+h(u,v_{start}),lmc^{i-1}(u)]\\
        \dot{<}&[lmc^{i}(v)+h(v,v_{start}),lmc^{i-1}(v)]\\
        \dot{<}&[g^{i-1}(v)+h(v,v_{start}),g^{i-1}(v)]=k^{i-1}(v),
    \end{split}
\end{equation*}
where the second inequality holds because of the triangular property of the heuristic function.
It indicates that $u$ was expanded before the i$^{th}$ iteration, and $v$ must have been rewired as $u$'s child node and inserted into $\mathcal{Q}$ before the i$^{th}$ expansion by the rewiring neighbor procedure (line \ref{ReduceInconssistency:RewireNeighborBegin}-\ref{ReduceInconssistency:RewireNeighborEnd}, Algorithm \ref{ComputeShortestPath}) or the $\FuncSty{UpdateNode}$ procedure invoked at line \ref{Main:RefreshUpdateEnd} of Algorithm \ref{Main} or line \ref{AddNewSample:UpdateNode} of Algorithm \ref{ExtendSearchGraph}, and has a smaller key than its current one, which is contradictory. 
Hence the assumption is impossible.    
\end{proof}

The following theorems can be derived from the above lemmas:
\begin{theorem}
\label{theorem_inconsistent}
If a node $u$ is inconsistent before the i$^{th}$ expansion, its key after the i$^{th}$ expansion $k^{i}(u)$ must be not smaller than $k^{i-1}(v)$.
\end{theorem}

\begin{proof}
If $u$'s key doesn’t change after the i$^{th}$ expansion, we have $k^{i}(u)=k^{i-1}(u)\dot{\geq} k^{i-1}(v)$ according to Lemma \ref{lemma_min_key}. 
In the opposite case where $u$'s key changes, $u$ must be rewired to $v$ and $lmc^i(u)\leftarrow lmc^i(v)+\bar{w}(u, v)$ (line \ref{ReduceInconssistency:RewireNeighborBegin}-\ref{ReduceInconssistency:RewireNeighborEnd} of Algorithm \ref{ComputeShortestPath} are executed).  
We consider two cases:

(1) If $v$ is overconsistent, $k^{i-1}(v)=k^{i}(v)=[lmc^i(v)+h(v,v_{start}),lmc^i(v)]$ according to Lemma \ref{lemma_overconsistent}.
According to the triangular property of the heuristic function, $lmc^i(v)+h(v, v_{start}) = lmc^i(u) - \bar{w}(u,v) + h(v, v_{start}) \leq lmc^i(u)+h(u, v_{start})$.
Based on Lemma \ref{lemma_min_key}, $lmc^{i-1
}(v)+h(v, v_{start}) < g^{i-1}(u)+h(u,v_{start})$ or $lmc^{i-1}(v) < g^{i-1}(u)$ if $lmc^{i-1
}(v)+h(v, v_{start}) = g^{i-1}(u)+h(u,v_{start})$.
Therefore, $k^i(u)=[min(lmc^i(u), g^{i-1}(u))+h(u,v_{start}), min(lmc^i(u), g^{i-1}(u))]\dot{\geq} k^{i-1}(v)$ is derived.

(2) If $v$ is underconsistent, $k^{i-1}(v)=[g^{i-1}(v)+h(v,v_{start}),g^{i-1}(v)]$ according to the definition of key in Section \ref{Methodology}. 
Since the lmc-value of $v$ after expansion $lmc^i(v)$ must be not smaller than $g^{i-1}(v)$ based on Lemma \ref{lemma_underconsistent}, we have $lmc^i(u)= lmc^i(v) + \bar{w}(u,v) \geq g^{i-1}(v) + \bar{w}(u,v)$. 
According to the triangular property,
$lmc^i(u) + h(u, v_{start}) \geq lmc^i(v)+h(v,v_{start})\geq g^{i-1}(v)+h(v,v_{start})$. 
Therefore, $k^i(u)\dot{\geq} k^{i-1}(v)$ is derived. 
\end{proof}

\begin{theorem}
\label{theorem_nondecreasing}
The key of the vertices expanded by the $\FuncSty{ComputeShortestPath}$ procedure is non-decreasing.
\end{theorem}
\begin{proof}
According to Lemma \ref{lemma_min_key}, its key is the smallest one among the other nodes in $\mathcal{Q}$. 
If a consistent vertex $u$ becomes inconsistent and is inserted into $\mathcal{Q}$ due to the expansion of $v$, it must have a larger key than $v$ as  
\begin{equation*}
\begin{split}
    k^{i}(u)=&[lmc^{i}(u)+h(u,v_{start}),lmc^{i}(u)]\\
    =&[lmc^{i}(v)+\bar{w}(v,u)+h(u,v_{start}),lmc^{i}(v)+\bar{w}(v,u)] \\
    \dot{>}&[lmc^{i}(v)+h(u,v_{start}),lmc^{i}(v)]\dot{\geq} k^{i}(v),
\end{split}
\end{equation*}
where the first inequality holds because of the triangular property of the heuristic function.
For other inconsistent nodes, their new key must be no smaller than the key of $v$ after the i$^{th}$ iteration according to Theorem \ref{theorem_inconsistent}. 
The theorem is proved. 
\end{proof}

\begin{theorem}
\label{theorem_overconsist}
If $v$ is overconsistent, it will be and remain to be consistent until $\FuncSty{ComputeShortestPath}$ terminates.   
\end{theorem}
\begin{proof}
$\FuncSty{ComputeShortestPath}$ makes $v$ consistent and pop it out of $\mathcal{Q}$ after the i$^{th}$ expansion.
Assuming that $v$ is rewired at the (i+j)$^{th}$ (j$>$0) iteration where a node $u$ is expanded, and is inconsistent again, we have $lmc^{i+j}(u)<lmc^{i}(v)$. 

If $u$ is overconsistent, $lmc^{i+j-1}(u)<lmc^{i}(v)+\bar{w}(v,u)$ must hold based on Theorem \ref{theorem_nondecreasing}, and therefore $lmc^{i+j-1}(u)+h(u, v_{start})<lmc^{i}(v)+h(u, v_{start})+\bar{w}(v,u)\leq lmc^{i}(v)+h(v, v_{start})$. Hence, $k^{i+j-1}(u)\dot{<}k^i(v)=k^{i-1}(v)$ must be true, which is contradict to Theorem \ref{theorem_nondecreasing}.

If $u$ is underconsistent, $lmc^{i+j-1}(u)>g^{i+j-1}(u)\geq lmc^{i+j-1}(v)$ must hold. 
Based on Lemma \ref{lemma_underconsistent}, $lmc^{i+j}(u)>g^{i+j-1}(u)$ and thus $lmc^{i+j}(u)>lmc^{i+j-1}(v)$. In this case, $v$ will not be rewired as the conditions of line \ref{ReduceInconssistency:RewireNeighborBegin} are not satisfied, which contradicts the assumption.
To sum up, if $v$ is overconsistent, it will not be rewired until the end of the $\FuncSty{ComputeShortestPath}$ procedure.
\end{proof}

\begin{theorem}
\label{theorem_consistent}
If a consistent node $u$ has a key $k^{i-1}(u)$ no greater than $k^{i-1}(v)$, it will not be rewired until the $\FuncSty{ComputeShortestPath}$ procedure terminates. 
\end{theorem}
\begin{proof}
Assume that $u$ is rewired due to the expansion of $v$. If $v$ is an overconsistent node, $k^{i-1}(v)=k^i(v)=[lmc^{i}(v)+h(v,v_{start}),lmc^i(v)]$ from Lemma \ref{lemma_overconsistent}.
The key of $u$ before expansion satisfies 
\begin{equation*}
\begin{split}
    k^{i-1}(u)=&[lmc^{i-1}(u)+h(u,v_{start}),lmc^{i-1}(u)]\\
    \dot{>}& [lmc^i(v)+\bar{w}(v,u)+h(u,v_{start}),lmc^i(v)]\\
    \dot{\geq}&[lmc^{i}(v)+h(v,v_{start}),lmc^i(v)]=k^{i-1}(v),
\end{split}
\end{equation*}
which is contradict to the condition.

If $v$ is an underconsistent node, based on Lemma \ref{lemma_underconsistent}, $lmc^{i}(v)$ is not smaller than $g^{i-1}(v)$, further deducing that $lmc^{i}(v)>lmc^{i-1}(u)$.
Hence, $u$ will not be rewired at the i$^{th}$ iteration.
From Theorem \ref{theorem_nondecreasing}, the following nodes that will be expanded after the i$^{th}$ iteration all have a larger key than $v$. 
In conclusion, $u$ will not be rewired until the $\FuncSty{ComputeShortestPath}$ procedure terminates.
\end{proof}

\begin{theorem}
\label{theorem_consistent_lmc}
If a consistent node $u$ has a key no greater than $v$, its lmc-value is equal to its minimum cost-to-goal w.r.t. the current graph.
\end{theorem}
\begin{proof}
According to Theorem \ref{theorem_consistent}, $u$ will not be expanded at the current $\FuncSty{ComputeShortestPath}$ procedure since there is no node that can decrease the cost-to-goal of $u$. 
The same rule also applies to the parent node of $u$. 
According to the shortest-paths optimality conditions (Section 4.4, \cite{Algorithm4}), if the cost-to-goal satisfies $lmc(u)\leq \min_{u'\in N(u)}lmc(u')+\bar{w}(u, u')$ for every node from $u$ to $v_{goal}$, $u$'s cost-to-goal is minimized, and the shortest path from $u$ to ${v}_{goal}$ can be obtained by tracing back the parent node from $u$.    
\end{proof}

\begin{theorem}
\label{theorem_consistent_once}
A consistent node will be expanded at most once till the $\FuncSty{ComputeShortestPath}$ procedure terminates. 
\end{theorem}
\begin{proof}
Without loss of generality, we assume $v$ is the first node expanded by $\FuncSty{ComputeShortestPath}$.
If a consistent node $u$ has a key no greater than $v$, according to Theorem \ref{theorem_consistent}, it will not be expanded till the termination of the $\FuncSty{ComputeShortestPath}$ procedure. 
If its key is greater than $v$, $u$ might be rewired as $v$'s child node to yield a smaller lmc-value and be overconsistent. 
According to Theorem \ref{theorem_overconsist}, an overconsistent node will be expanded at most once till the end of $\FuncSty{ComputeShortestPath}$. 
In all, a consistent node will be expanded at most once throughout the entire $\FuncSty{ComputeShortestPath}$ procedure.
\end{proof}

\begin{theorem}
\label{theorem_inconsistent_twice}
An inconsistent node will be expanded at most twice till the $\FuncSty{ComputeShortestPath}$ procedure terminates. 
\end{theorem}
\begin{proof}
If $v$ is overconsistent, based on Theorem \ref{theorem_overconsist}, $v$ will be expanded at most once. 
If $v$ is underconsistent, there is two possible cases:
(1) $v$ finds a parent. After expansion, $lmc^{i}(v)=g^{i}(v)<\infty$. $v$ will be popped out as a consistent node. 
According to Theorem \ref{theorem_consistent_once}, it will be expanded at most once. 
(2) $v$ doesn’t find a parent. 
After expansion, $lmc^{i}(v)=g^{i}(v)=\infty$. $v$ will be popped out as an underconsistent node. It will only be updated and reinserted into $\mathcal{Q}$ as an overconsistent node if one of its neighbor nodes finds a path to the goal and is expanded.
According to Theorem \ref{theorem_overconsist} and \ref{theorem_consistent}, it will be expanded at most once. 
In all, $v$ will be expanded at most twice till the $\FuncSty{ComputeShortestPath}$ procedure terminates. 
The result can be applied to all inconsistent nodes. 
The proof is complete.
\end{proof}

Now, we are ready to prove Theorem \ref{theorem_optimality}.
From Theorem \ref{theorem_consistent_once} and Theorem \ref{theorem_inconsistent_twice}, we know that every node in $V$ will be expanded at most twice by $\FuncSty{ComputeShortestPath}$.
As there is a finite number of nodes in $V$, the loop will eventually break due to $|\mathcal{Q}|=0$ (line 1, Algorithm \ref{ComputeShortestPath}).
If the loop breaks when the following conditions are no longer satisfied: (1) $lmc(v_{start})>g(v_{start})$ or $lmc(v_{start})=g(v_{start})=\infty$; (2) $v_{start}\in\mathcal{Q}$; 
(3) $\mathcal{Q}.topkey()<k(v_{start})$,
it indicates that $v_{start}$ is a consistent node with a key no greater than the key of the top node in $\mathcal{Q}$. 
From Theorem \ref{theorem_consistent}, we know that the lmc-value of $v_{start}$ is equal to its minimum cost-to-goal w.r.t. the current search graph $\mathcal{G}$.
If the loop breaks with $lmc(v_{start})=\infty$, which only happens when $|\mathcal{Q}|=0$, suggesting that every inconsistent node that is generated before or during the $\FuncSty{ComputeShortestPath}$ procedure is expanded. 
However, none of them can rewire $v_{start}$ as their child node, indicating that either $v_{start}$ can not be reached from them via neighboring search, or the inconsistent nodes that can be reached by $v_{start}$ do not have lmc-value smaller than $\infty$, i.e., they can not find a path to the goal, and therefore the line \ref{ReduceInconssistency:RewireCondition} of Algorithm \ref{ComputeShortestPath} is not satisfied and the rewiring neighbor procedure can not be triggered. 
Therefore, there is no solution path from $v_{start}$ to $v_{goal}$ w.r.t. $\mathcal{G}$.
The proof is complete.

\end{document}